\documentclass[lettersize,journal]{IEEEtran}

\usepackage[T1]{fontenc}

\ifCLASSINFOpdf
\else
\fi
\usepackage{cuted}

\usepackage{amsmath}
\usepackage[cmintegrals]{newtxmath}

\usepackage[belowskip=-15pt,aboveskip=0pt]{caption}
\usepackage{subcaption}
\usepackage{subfloat}
\usepackage{amsthm}
\usepackage{setspace}
\usepackage{float}
\usepackage{lipsum}
\usepackage{amsmath,amsfonts}

\usepackage{amssymb}
\usepackage{graphicx}
\usepackage[colorinlistoftodos]{todonotes}
\usepackage[ruled,vlined]{algorithm2e}
\usepackage{scalerel}
\usepackage{tabularx}
\usepackage[inline]{enumitem}
\usepackage{diagbox}

\theoremstyle{definition}

\newtheorem{theorem}{Theorem}
\newtheoremstyle{exampstyle}
{1pt} 
{1pt} 
{} 
{} 
{\bfseries} 
{} 
{.5em} 
{} 

\theoremstyle{exampstyle} \newtheorem{example}{Example}
\theoremstyle{exampstyle} \newtheorem{remark}{Remark}
\theoremstyle{exampstyle} 
\theoremstyle{exampstyle} \newtheorem{lemma}{Lemma}
\theoremstyle{exampstyle} \newtheorem{assumption}{Assumption}

\usepackage{xpatch}
\makeatletter
\xpatchcmd{\@thm}{\thm@headpunct{.}}{\thm@headpunct{}}{}{}
\makeatother

\begin{document}
\title{Hierarchical Federated Learning with Quantization: Convergence Analysis and System Design}

\author{Lumin~Liu,~\IEEEmembership{Student~Member,~IEEE,}
	Jun~Zhang,~\IEEEmembership{Fellow,~IEEE,}
	Shenghui~Song,~\IEEEmembership{Member,~IEEE,}
	and~Khaled~B.~Letaief,~\IEEEmembership{Fellow,~IEEE}
\thanks{Part of the results was presented at the IEEE International Conference on Communications (ICC), 2020 \cite{9148862}. 
{
This work was supported by the General Research Fund (Project No. 16208921) from the Research Grants Council of Hong Kong.} The authors are with the Department of Electronic and Computer Engineering, Hong Kong University of Science and Technology, Hong Kong (Email: {lliubb, eejzhang, eeshsong, eekhaled} @ust.hk). (The corresponding author is J. Zhang.)}
}

\maketitle

\begin{abstract}
Federated learning (FL) is a powerful distributed machine learning framework where a server aggregates models trained by different clients without accessing their private data. Hierarchical FL, with a client-edge-cloud aggregation hierarchy, can effectively leverage both the cloud server's access to many clients' data and the edge servers' closeness to the clients to achieve a high communication efficiency. Neural network quantization can further reduce the communication overhead during model uploading. To fully exploit the advantages of hierarchical FL, an accurate convergence analysis with respect to the key system parameters is needed. Unfortunately, existing analysis is loose and does not consider model quantization. In this paper, we derive a tighter convergence bound for hierarchical FL with quantization. The convergence result leads to practical guidelines for important design problems such as the client-edge aggregation and edge-client association strategies. Based on the obtained analytical results, we optimize the two aggregation intervals and show that the client-edge aggregation interval should slowly decay while the edge-cloud aggregation interval needs to adapt to the ratio of the client-edge and edge-cloud propagation delay. Simulation results shall verify the design guidelines and demonstrate the effectiveness of the proposed aggregation strategy.
\end{abstract}


%
\IEEEpeerreviewmaketitle

\section{Introduction}
\par 
Federated Learning (FL) \cite{mcmahan2017communication} is a promising and powerful training framework for privacy-preserving Deep Learning (DL), where clients train their models locally and then upload to a parameter server for model aggregation. This process is repeated for many rounds until the aggregated model at the server reaches a target accuracy.  With FL, model training occurs at the clients and only the trained models are required to be aggregated at the server. This eliminates the need for sharing private user data and thus preserves data privacy. The feasibility of FL has been verified in real-world DL applications, e.g., the Google keyboard prediction \cite{Keyboardprediction18Hard}. However, the long propagation delay, the increasing size of the DL models and limited resources make  communication efficiency one of the most critical challenges in FL. Specifically, hundreds to thousands of rounds of communications may be required to reach a desired model accuracy. Furthermore, the overhead of one round of communication, including upload and download, is proportional to the model size, making it {inefficient for employing large models, especially when the clients are connected to the server by wireless links}.

\par 
Most initial studies on FL assumed one cloud node as the parameter server, while some recent works \cite{tran2019federated,wang2019adaptive} proposed to address the communication challenge by leveraging Mobile Edge Computing (MEC) platforms \cite{MECSurveyMao}. This transition led to Federated Edge Learning (FEL) \cite{EdgeAIKhaled}, which enables FL at the network edge to support ultra-low latency applications \cite{intelligentedgejun,8884164}. Although edge-based FL enjoys a lower round-trip latency, the number of clients that can participate in the training process is small, which degrades the training performance. Thus, a trade-off between the round-trip latency and learning performance exists when choosing between the cloud server and the edge server for the two-layer architecture.
This trade-off motivated a hierarchical architecture for FL \cite{9148862,castiglia2021multilevel,zhou2019distributed}, which includes one cloud server, multiple edge servers and many clients. There are two levels of aggregation in hierarchical FL, namely, the efficient and parallel edge aggregation and the time-consuming cloud aggregation. It was shown in our previous work \cite{9148862}, both theoretically and empirically, that hierarchical FL achieves a faster convergence speed than the conventional two-layer cloud-based FL. Furthermore, empirical experiments showed that the hierarchical architecture reduces training time compared with the single-server architecture. 
\par
Quantized message passing during the aggregation stage is another important technique to improve the communication efficiency and presents a similar communication-accuracy trade-off. The size of a full-precision model can be up to hundreds of megabytes for large DL models. By transmitting a low-precision approximate of the model, the communication cost can be greatly reduced. However, it is noted that an appropriate level of quantization can accelerate the training while a very coarse approximation will fail the training. Thus, model quantization shares a similar communication-accuracy trade-off as the hierarchical structure. As such, we shall investigate them together in this paper. Typical quantization techniques include low-rank approximation \cite{wang2018atomo}, sparsification \cite{NEURIPS2018_3328bdf9} and low-bit quantization \cite{NIPS2017_6c340f25}. In the two-layer FL system, both theoretical analyses and experimental results have shown that quantization can significantly improve the communication efficiency  \cite{fasttcom20,konevcny2016federated,reisizadeh2020fedpaq}. However, the associated analysis for hierarchical FL is not available yet.


\par
To fully exploit the potential of these communication-efficient techniques and architectures, we need an accurate convergence analysis based on which system optimization can be performed \cite{tran2019federated, schedulingICCNiu}.  For example, when minimizing the communication cost in a FL system, a typical and critical question to ask is how to optimize the aggregation interval for a given training time budget $T$. The solution of this problem highly depends on the variance term caused by the local steps in the derived convergence upper bound \cite{wang2019adaptive,fasttcom20,MLSYS2019Wang}. 
{Thus, deriving tighter convergence bounds is of crucial importance not only to the theoretical understanding, but also to practical system design. There have been many recent efforts on this aspect, e.g.,  \cite{khaled2020tighter, pmlr-v130-haddadpour21a}, but the focus is on the server-based two-layer structure, and the investigation.}
More importantly, the theoretical understanding of hierarchical FL is primitive. Existing analyses and system design for the hierarchical FL system \cite{castiglia2021multilevel,zhou2019distributed, wang2020local} were carried out assuming full-precision model updates, and the obtained convergence bound was loose \cite{castiglia2021multilevel, zhou2019distributed}, which made the optimization formulation inaccurate. The analysis developed in our previous work \cite{9148862} has a complex expression and thus is difficult to be utilized for further system optimization. Thus, getting a sharp and tighter bound for hierarchical quantized FL is critical.
\par
In this paper, we consider a communication-efficient hierarchical FL system, where two levels of aggregation and quantization are adopted. The aim is to provide an accurate theoretical analysis to support further system design. Based on the analytical results, system design guidelines will be provided and an adaptive aggregation interval selection scheme is proposed.

\subsection{Related Works}

The convergence of the two-layer FL has been well established for convex and non-convex loss functions in \cite{stich2019local, khaled2020tighter, wang2018cooperative}. In \cite{khaled2020tighter, wang2018cooperative}, the additional error term caused by multiple local updates was shown to grow linearly with the aggregation interval and this is the tightest  error bound proved so far in the literature. In \cite{reisizadeh2020fedpaq}, the convergence of the FedAvg algorithm considering a random client selection strategy and model quantization was also investigated{, where the communication cost is reduced with client partial participation and model quantization}. 
{Recently, Federated Dropout \cite{wen2021feddrop} was proposed. Such a method adopts a random model pruning approach, i.e., dropout, to reduce the local computation and communication cost. Clearly, this is another direction to improve the communication efficiency and can be combined with model quantization methods to further reduce the communication cost.}
\par
System optimization is critical to improve the communication efficiency of the FL system. Current research for system design mainly focused on the following three aspects: spectrum allocation, power control \cite{tran2019federated}, and local aggregation interval control \cite{wang2019adaptive, MLSYS2019Wang}. In these works, the system design was formulated as an optimization problem to minimize the latency of the FL system subject to a given accuracy constraint. 
\par 
The convergence analysis for hierarchical FL has been less well studied. Our previous work \cite{9148862} analyzed the convergence of hierarchical FL with the full-precision model update for both convex and non-convex loss functions. However, the optimizer at the client side is a full-batch gradient descent, which may not be practical for devices with a limited computation power. In \cite{zhou2019distributed}, the authors provided a convergence analysis of hierarchical FL for non-convex loss functions, where the obtained error bound is quadratic with the aggregation interval. Later, a tighter error bound was provided in \cite{wang2020local} for non-convex loss functions. The analyses in \cite{zhou2019distributed, wang2020local} both considered full-precision model updates. In addition to the theoretical convergence analysis, a typical and unique system design problem for hierarchical FL, i.e., the edge-client association and resource allocation problem, were investigated in \cite{luo2020hfel,9629331},. 

\subsection{Contributions}
This paper investigates a hierarchical FL system with model quantization. A communication-efficient training algorithm, Hier-Local-QSGD, is proposed, where clients upload quantized updates to their associated edge servers after $\tau_1$ steps of local updates, and the edge servers upload quantized updates to the cloud server after $\tau_2$ steps of edge aggregation. We summarize the paper contributions as follows:

\begin{itemize}
	\item A tighter convergence bound of the Hier-Local-QSGD algorithm is provided for non-convex loss functions, where the variance caused by the aggregation interval is reduced from a quadratic term to a linear term. We show that after $K$ communication rounds with the cloud server, i.e., $K\tau_1\tau_2$ local update iterations, Hier-Local-QSGD converges to a first-order stationary point at a rate of $\mathcal{O}(1/\sqrt{K\tau_1\tau_2})$.
	\item The obtained analytical result leads to two interesting design guidelines:
	\begin{enumerate*}
		\item when the quantization variance exceeds a threshold, infrequent client-edge aggregation is preferred; and 
		\item the edge-client association strategy has no influence on the convergence for a given number of clients and edge servers. 
	\end{enumerate*} 
	\item The derived convergence upper bound is utilized to formulate the aggregation interval selection problem, which is solved by an adaptive aggregation interval control algorithm. It will be shown that the client-edge aggregation interval $\tau_1$ is expected to decay at a rate that depends on the training loss while the edge-cloud aggregation interval $\tau_2$ is determined by the ratio of the propagation delay and the number of edge servers and clients.
\end{itemize}
To the best of our knowledge, this is the first work that proves the linear error term with respect to the local update, as well as presenting a system design for hierarchical quantized FL.

\par 
The rest of this paper is organized as follows. In Section \ref{System Description}, we will introduce the learning problem in FL, the hierarchical FL system, and the  Hier-Local-QSGD algorithm. In Section \ref{Convergence Analysis}, we present the convergence analysis with a sketch of the proof while a detailed proof can be found in the appendix. Discussions on the convergence results are included to provide a connection between the analysis and system design guidelines. In Section \ref{Applications}, we will introduce the adaptive aggregation interval selection scheme. In Section \ref{Simulations}, empirical results, as applied on the CIFAR-10 dataset, are presented. We will also verify the two guidelines of the system design and the effectiveness of the adaptive aggregation interval selection scheme.

\section{System Description} \label{System Description}

In this section, we will introduce the FL learning problem and different FL architectures. The standard single-server FL system and its algorithm will be briefly reviewed. We will then introduce the hierarchical FL system with one cloud server, $s$ edge servers, $n$ users, and the corresponding Hier-Local-QSGD training algorithm. 
{
\subsection{FL Problem}
For FL, suppose that there are $n$ clients, and the $i$-th client is with dataset $\{ \mathcal{D}_i \}$ of size $D_i$. Based on the local dataset $\{ \mathcal{D}_i \}$, the empirical local loss function for the $i$-th user is $f_i (x) = \frac{1}{D_i} \sum_{{\xi_j}\in \mathcal{D}_i}\mathcal{L}(x,\xi_j) \label{eq:4}.$
The goal of the FL training algorithm is to learn a global model that performs well on the joint data distributions. Denote the joint dataset as $\mathcal{D} = \bigcup_{i=1}^n \mathcal{D}_i $. Then, the final loss function that needs to be minimized is given by $f (x) = \frac{1}{\sum_{i=1}^{n}D_i} \sum_{{\xi_j}\in \mathcal{D}}\mathcal{L}(x,\xi_j) = \frac{1}{\sum_{j=1}^{n}D_j} \sum_{i=1}^n D_i f_i(x).$
}

\begin{table}[t]
	\centering
	\small
	\caption{Key Notations for the Hier-Local-QSGD algorithm}
	\vspace{10pt}
	\label{table:1}
	\begin{tabular}{ | m{5em} | m{5.9cm}| } 
		\hline
		\textbf{Symbol}& \textbf{Definitions} \\[5pt]
		\hline
		$x$ & $x \in \mathbb{R}^p$, parameters of the learning model \\ [5pt]
		\hline
		$x_t$ & Model parameters after total $t$ steps of local updates \\ [5pt]
		\hline
		$n$ & The number of participated clients in the whole hierarchical FL system \\ [5pt]
		\hline
		$s$ & The number edge servers in the whole system \\ [5pt]
		\hline
		$m^ \ell$ & The number of participated clients under edge parameter server $\ell$  \\ [5pt]
		\hline
		$\mathcal{C}^\ell$ & The set of the clients under edge $\ell$ \\ [5pt]
		\hline
		$\tau_2$ & The edge-cloud aggregation interval \\[5pt] 
		\hline
		$\tau_1$ & The client-edge aggregation interval   \\ [5pt]
		\hline
		$k$ & The index of the cloud aggregation round \\ [5pt]
		\hline
		$t_1$ & The index of the client local update step from the last edge aggregation, $0 \leq t_1< \tau_1$ \\ [5pt]
		\hline
		$t_2$ &The index of the edge-aggregation round from the last cloud aggregation, $0 \leq t_2 < \tau_2 $ \\ [5pt]
		\hline
		$x_{k, t_2, t_1}$ &Local model parameters on client $i$ at local update step $(k, t_2, t_1)$, where the total number of local updates $t=k\tau_1\tau_2+ t_2\tau_1+t_1$ \\ [5pt]
		\hline
		$u^{\ell}_{k,t_2}$ & Edge model parameters on edge $\ell$ after the $t_2$-th edge aggregation in the $k$-th cloud aggregation interval  \\ [5pt]
		\hline
		$x_k$&  Cloud model parameters after the $k$-th cloud aggregation \\ [5pt]
		\hline
		
	\end{tabular}
\end{table}

\subsection{Training Algorithm of Two-Layer FL and Hierarchical FL}
 
In the two-layer FL system, there is one central parameter server and $n$ clients. Each client performs $\tau$ steps of SGD iterations locally and then uploads the model updates to the central parameter server. The central server averages the updates and redistributes the averaged outcomes back to each client. The process repeats until the model reaches a desired accuracy or the limited resources, (e.g., the energy or time budget) run out. 
\par
The parameters of the local model on the $i$-th client after $t$ steps of SGD iterations are denoted as $x_t^i$. In this case,  $x_t^i$ in the FedAvg algorithm evolves in the following way: 

\begin{equation}
\text{$x_t^i$ } = 
\begin{cases}
\text{$x_{t-1}^i - \eta  \tilde{\nabla} f_i(x_{t-1}^i)$} &  \text{$t \mid \tau \neq 0$}\\
\text{ $\frac{1}{n} \sum_{i=1}^n[x_{t-1}^i - \eta \tilde{\nabla} f_i(x_{t-1}^i)]$ } &
\text{$t \mid \tau = 0$}
\end{cases} \label{eq:6}
\end{equation}

\par
In FedAvg, the model aggregation step can be interpreted as a way to exchange information among the clients. Thus, the aggregation at a cloud parameter server can incorporate data from many clients, but the communication cost is high. On the other hand, aggregation at an edge parameter server only incorporates a small number of clients with a much cheaper communication cost. 
\par
To combine these advantages, a hierarchical FL system is considered, which has one cloud server, $s$ edge servers indexed by $\ell$, with disjoint client sets $\{\mathcal{C}^\ell\}_{\ell=1}^s$, and $n$ clients indexed by $i$ and $\ell$, with distributed datasets $\{\mathcal{D}_i^\ell\}_{i=1}^N$.  Other key notations that are important for the algorithm and the theoretical analysis are summarized in Table \ref{table:1}. The hierarchical FL system exploits the natural client-edge-cloud communication hierarchy in current communication networks.

\par
With this hierarchical FL architecture, we propose a Hier-Local-QSGD algorithm as described in Algorithm \ref{algorithm1}. The key steps of the Hier-Local-QSGD algorithm include the following two modules to improve communication efficiency.

\begin{algorithm}[t] 
	\setstretch{1}
	\SetAlgoLined
	Initialize the model on the cloud server $x_0$\;
	\For{$k = 0,1,\dots, K-1$}{
		\For{$\ell = 1,\dots,s$ \text{edge servers in parallel}}{
			Set the edge model same as the cloud server model\;
			$u_{k,0}^\ell = x_k$\;
			\For{$t_2 = 0,1,\dots,\tau_2-1$}{
				\For{$i \in \mathcal{C}^\ell$ \text{clients in parallel}}{
					Set the client model same as the associated edge server model\;
					$x_{k,t_2,0}^i= u_{k,t_2}^\ell$\;
					\For{$t_1 = 0,1,\dots,\tau_1-1$}{
						$x_{k,t_2,t_1+1}^i = x_{k,t_2,t_1}^i - \eta \Tilde{\nabla}f_i(x_{k,t_2,t_1}^i)$
					}
					Send $Q_1(x_{k,t_2\tau_1}^i - x_{k,t_2,0}^i)$ to its associated edge server
				}
				Edge server aggregates the quantized updates from the clients\;
				$u_{k,t_2+1}^\ell = u_{k,t_2}^\ell + \frac{1}{m^\ell}\sum_{i\in \mathcal{D^\ell}} Q_1(x_{k,t_2\tau_1}^i - x_{k,t_2,0}^i)$
			}
			Send $Q_2(u_{k,\tau_2}^\ell - u_{k,0}^\ell)$
		}
		Cloud server aggregates the quantized updates from the edge servers\;
		$x_{k+1} = x_k + \sum_{\ell=1}^s  \frac{m^\ell}{n} Q_2(u_{k,\tau_2}^\ell - u_{k,0}^\ell)$
	}
	\caption{Hierarchical Local SGD with Quantization (Hier-Local-QSGD)}
	\label{algorithm1}
\end{algorithm}

\subsubsection{\textbf{Frequent Edge Aggregation and Infrequent Cloud Aggregation}}
\par
Periodic aggregation is the key step in reducing the communication cost in FL. A larger aggregation interval, $\tau$, reduces the communication rounds. But a large $\tau$ will degrade the performance of the obtained DL model. This is because too many steps of local SGD updates will lead the local models to approach the optima of the local loss function $f_i(x)$ instead of the global loss function $f(x)$. 
\par
Edge aggregation enjoys a lower propagation latency compared with cloud aggregation. Hence, in Hier-Local-QSGD, each edge server efficiently aggregates the models within its local area for several times before the cloud aggregation. To be more specific, after every $\tau_1$ local SGD updates on each client, each edge server averages its clients' models. After every $\tau_2$ edge aggregations, the cloud server averages all the edge servers' models. Thus, the communication with the cloud happens after every $\tau_1 \tau_2$ local updates. In this way, the local model is less likely to be biased towards its local minima compared with the case in FedAvg with an aggregation interval of $\tau=\tau_1\tau_2$.

\subsubsection{\textbf{Quantized Model Updates}}
The overall communication cost in FL also depends on the DL model size, which determines the amount of data to be transmitted in each communication round. Quantization is often used to reduce the size of the model updates. A low-precision quantizer reduces the communication overhead but introduces additional noise during the training process, which will ultimately degrade the trained model performance. Thus, investigating the effect of quantization is important.

\begin{figure*}[t] 
	\centerline{\includegraphics[width=.85\linewidth]{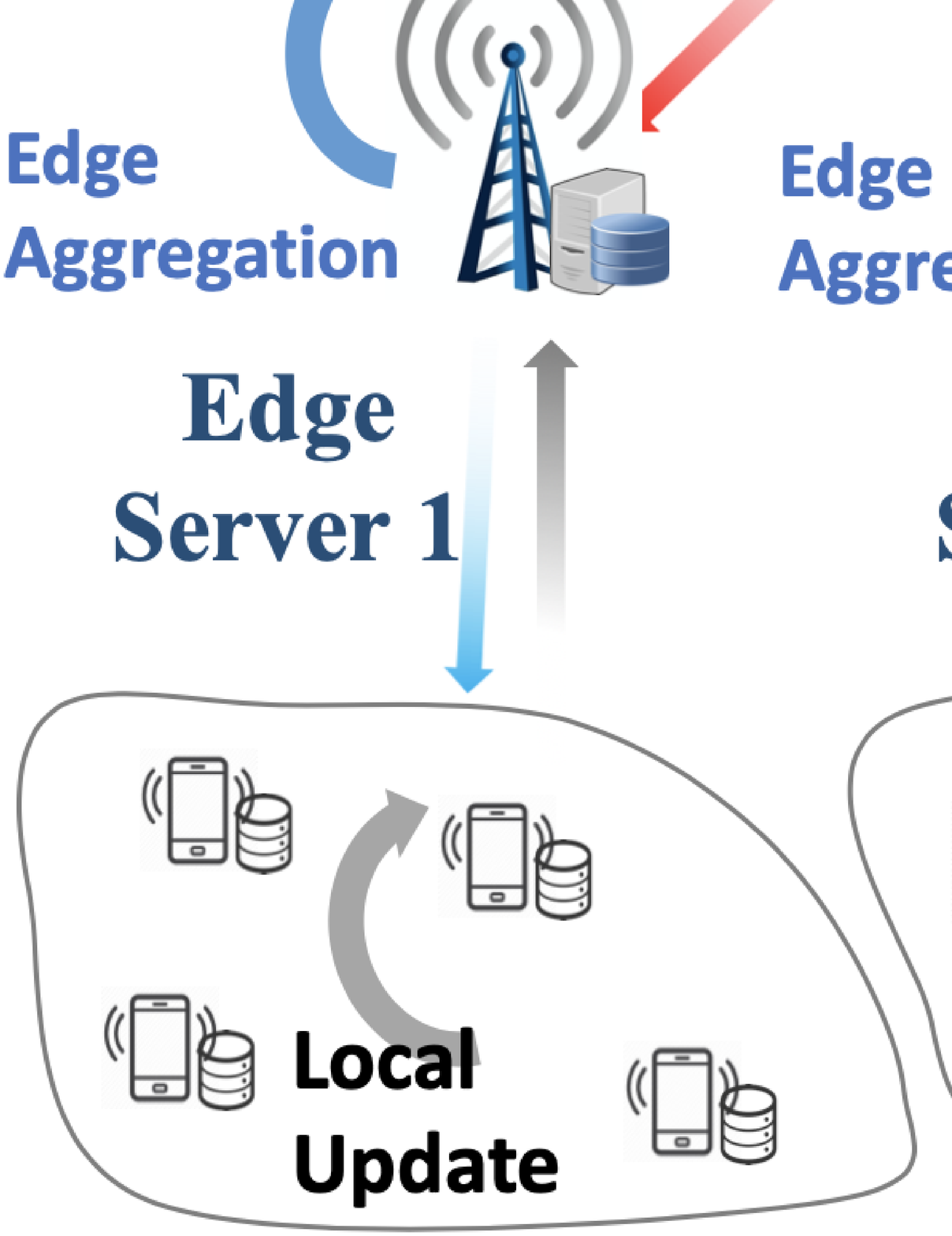}} 
	\vspace{20pt}
	\caption{Illustration of the hierarchical architecture and Hier-Local-QSGD algorithm.}
	\label{fig2}
\end{figure*}

\par
Here, we give two examples of widely-used random quantizers, i.e., random sparsification \cite{wang2018atomo} and stochastic rounding \cite{NIPS2017_6c340f25}.
\begin{example}{(Random Sparsification)}\label{example1}
	For any $\boldsymbol{x} \in \mathbb{R}^n$, fix $r \in {1,\dots,d}$ and let $\zeta \in \mathbb{R}^d$ be a (uniformly distributed) random binary vector with $r$ non-zero entries. The random sparsification operator is given by:
	\begin{equation*}
		Q(x) = \frac{d}{r}\left(\zeta \odot \boldsymbol{x} \right)
	\end{equation*} 
	where $\odot$ denotes the Hadamard (entry-wise) product.
\end{example}
{
\begin{example}{(Stochastic Rounding)} \label{example2}
	For any $\boldsymbol{x} \in \mathbb{R}^n$ with $\boldsymbol{x} \neq 0$, stochastic rounding ${Q}_{s}(\boldsymbol{x})$ is defined as
	\begin{equation}
		Q_{s}(x_i) = \|\boldsymbol{x} \|_2 \cdot sgn(x_i) \cdot \xi_i(\boldsymbol{x}, s),
	\end{equation}	
	where $\xi_i(\boldsymbol{x}, s)$'s are independent random variables defined as follows, and $s \geq 1$ is a tuning parameter, corresponding to the quantization levels. Let $0 \leq \ell < s$ be an integer such that $|x_i|/\|\boldsymbol{x}_2 \| \in [ \ell/s, (\ell+1)/s ]$. Then,
	\begin{equation*}
		\xi_i(\boldsymbol{x}, s) = 
		\begin{cases}
		\ell/s& \text{with probability } 1-(\frac{|x_i|}{\|\boldsymbol{x} \|}, s); \\
		(\ell+1)/s & \text{otherwise.}
		\end{cases}
	\end{equation*}
	Here, $p(a, s) = as -\ell$ for any $a\in [0,1]$. If $\boldsymbol{x} = 0$, then we define $Q(\boldsymbol{x}, s)  =0$
\end{example}
}
\par
We use $Q_1(\cdot)\text{ and } Q_2(\cdot)$ to represent the specific quantizers applied on the model updates from the client to the edge server and the model updates from the edge servers to the cloud server, respectively.

\par
The system architecture and algorithm flow are illustrated in Fig. \ref{fig2}. A comparison between the FedAvg and Hier-Local-QSGD algorithms is also included in Fig. \ref{fig2}. The details of the Hier-Local-QSGD algorithm are presented in Algorithm \ref{algorithm1}. $x_{k,t_2,t_1}^i$ denotes the local model parameters after $k$ rounds of cloud aggregation, $t_2$ rounds of edge-aggregation and $t_1$ steps of local update on client $i$. Specifically, the total steps of the local iterations $t$ can be expressed as $t = k\tau_1\tau_2 + t_2 \tau_1 + t_1$. We will use the tuple $(k,t_2,t_1)$ for indexing throughout the paper. Similarly, the model parameters on edge $\ell$ after $k$ rounds of cloud aggregation and $t_2$ rounds of edge aggregation are denoted by $u^\ell_{k,t_2}$. Finally, the model parameters on the cloud server after $k$ rounds of cloud aggregation are denoted by $x_k$.
\par
The evolution of the model parameters $x_{k,t_2,t_1}^i, u^\ell_{k,t_2} \text{and } x_k$ can be described follows:
\begin{equation}
\begin{cases}
\small\text{Local Update: } &x_{k,t_2,t_1+1}^i = x_{k,t_2,t_1}^i - \eta \tilde{\nabla} f_i(x_{k,t_2,t_1}^i),\\
 &0 \leq t_1<\tau_1, 0 \leq t_2<\tau_2\\
\small\text{Edge Aggregation: }& x^i_{k, t_2+1,0} =u^\ell_{k,t_2+1} \\
&= u^\ell_{k,t_2} + \frac{1}{m^\ell} \sum_{i \in \mathcal{C}^\ell_i}[Q_1(x_{k,t_2,\tau_1}^i - x_{k,t_2,0}^i)],  \\
& t_1 = \tau_1,0 \leq t_2<\tau_2\\
\small\text{Cloud Aggregation: }& x^i_{k+1,0,0} = u^{\ell}_{k+1,0} = x_{k+1}\\ 
&= x_k + \sum_{\ell=1}^s  \frac{m^\ell}{n} Q_2(u_{k,\tau_2}^\ell - x_k), \\
&t_1=\tau_1, t_2=\tau_2
\end{cases} \label{eq:7}
\end{equation}

\section{Convergence Analysis} \label{Convergence Analysis}
In this section, we present the convergence analysis of the Hier-Local-QSGD algorithm for non-convex loss functions, followed by discussions of the main findings from the obtained convergence bound. We provide a sketch of the proof in this section, while ea detailed proof of the key lemmas can be found in the appendix.
\par
We first present three customary assumptions that are required for the analysis.
\par
\begin{assumption}[L-smoothness] \label{assumption1}
	The loss function $f(x): \mathbb{R}^p  \rightarrow \mathbb{R}$ is $L$-smooth with the Lipschitz constant $L >0$, i.e. :
	\begin{equation*}
	\Vert \nabla f(x) - \nabla f(y) \Vert \leq L \Vert x - y \Vert,
	\end{equation*}
	for all $x, y \in \mathbb{R}^p$.
\end{assumption}

\begin{assumption}[Variance of SGD] \label{assumption2}
	For any fixed model parameter $x$, the locally estimated stochastic gradient $\tilde{ \nabla} f_i(x)$ is unbiased and its variance bounded by a constant $\sigma^2$ for any client. That is,
	\begin{equation*}
	\begin{split}
	&\mathbb{E} [\tilde{ \nabla} f_i(x) |x] = \nabla f(x), \\
	&\mathbb{E} [\left \Vert \tilde{ \nabla} f_i(x) - \nabla f(x) \right \Vert^2 |x ]\leq \sigma^2.
	\end{split}
	\end{equation*}
\end{assumption}

\begin{assumption}[Unbiased Random Quantizer] \label{assumption3}
	The random quantizer $Q(\cdot)$ is unbiased and its variance grows with the squared $\ell_2$ norm of its argument. That is,
	\begin{equation*}
	\begin{split}
	&\mathbb{E} [ Q(x) | x ] = x, \\
	&\mathbb{E} [ \left \Vert Q(x)-x \right \Vert ^2 | x ] \leq q \left \Vert  x \right \Vert ^2.
	\end{split}
	\end{equation*}
	for some positive real constant $q$ and any $x \in \mathbb{R}^p $.
\end{assumption}
\par
Assumption \ref{assumption1} requires that the local loss function to be $L$-smooth, which means that its gradient is $L$-continuous. Assumption \ref{assumption2} on the bias and variance of stochastic gradients is common in the literature \cite{khaled2020tighter, castiglia2021multilevel}. This is quite a weaker assumption compared with the one that uniformly bounds the expected norm of the gradients \cite{zhou2019distributed}. Assumption \ref{assumption3} makes sure that the output of the random quantizer is unbiased and the variance is proportional to the norm of the input. Intuitively, a quantizer with a larger $q$ means a smaller communication cost and low-precision output. {This assumption is satisfied for many commonly used quantization schemes such as the random sparsification introduced in Example \ref{example1}, stochastic rounding in Example \ref{example2}, and the high-dimensional vector quantizer \cite{du2020high}.} For different quantization techniques, the variance parameter $q$ is known in literature and only depends on the quantization
technique itself. Thus, our analysis is applicable to many different quantization techniques by substituting the quantization variance parameter $q$ with the specific quantization parameters. 

\subsection{Convergence Analysis Challenges}  
In the following, we highlight the main challenges in the convergence analysis.
\begin{enumerate}
	\item {Two levels of aggregation}: While the local aggregation at the edge servers can incorporate partial information on the global loss function in a communication-efficient manner, it results in possible gradient divergence at different edge servers. This poses a major challenge in the analysis compared to the FedAvg algorithm \cite{stich2019local,wang2018cooperative}.
	\item{Model Uploading Compression}: The quantization of the local model weights for efficient model uploading will introduce errors in the training process. When there are several partial edge aggregations before a global cloud aggregation, the quantization error caused by the communication between the clients and edge server will accumulate. This indicates that in the final expression of the variance term, the influence of the quantization level and the aggregation interval are intertwined.
	\item {Tightness of the upper-bound}: There exists an analysis of the hierarchical local SGD for non-convex loss functions, e.g., \cite{castiglia2021multilevel,zhou2019distributed}, but the available bound is rather loose. A tighter analysis requires a novel approach since we need to focus on the evolution process of the local model parameter and carefully use inequalities to compute the upper bound. Besides, we do not require the gradient norm to be bounded, which is a relaxed assumption in \cite{zhou2019distributed}. 
	The lack of the uniformly bounded gradient norm assumption makes it difficult to bound the model parameter divergence in the evolution process.
\end{enumerate}
\textbf{Convergence Criterion}: For the error-convergence analysis for non-convex loss functions, the expected gradient norm is often used as an indicator of the convergence \cite{MLSYS2019Wang, bottou2018optimization}.  Specifically, an algorithm achieves $\epsilon$-suboptimal for a given positive value $\epsilon$ if 
$
	\mathbb{E} \left[ \min_{k=0,\dots, K -1} \left \Vert \nabla f(x_k) \right \Vert^2  \right] \leq \epsilon.
$
When $\epsilon$ is arbitrarily small, the algorithm converges to a first-order stationary point.

\subsection{Main Result and Discussions}

The following theorem presents the main convergence result.
\begin{theorem}[Convergence of Hier-Local-QSGD for non-convex loss functions] 
	\label{theorem1}
	Consider the sequence of iterations $ \left\{x_k \right\}$ at the cloud parameter server according to the Hier-Local-QSGD in Algorithm \ref{algorithm1}. Suppose that Assumptions \ref{assumption1}, \ref{assumption2}, \ref{assumption3} are satisfied, and the loss function $f$ is lower bounded by $f^*$. Further, define $G$ as:
	\begin{align}
	G = &1-L^2\eta^2\left[\frac{\tau_1(\tau_1-1)}{2}+\tau_1\tau_2\left(\frac{\tau_2(\tau_2-1)}{2}+q_1\tau_2 \right) \right] \notag \\
	&- L\eta(1+q_2) \left( \tau_1\tau_2 +\frac{q_1\tau_1}{n} \right),
	\end{align}
	where $q_1$ is the quantization variance parameter for the quantization operator at the client ($Q_1$ in Algorithm \ref{algorithm1}), $q_2$ is the quantization variance parameter for the quantization operator at the edge server ($Q_2$ in Algorithm \ref{algorithm1}), and $K$ is the total number of cloud communication rounds. When $G \geq 0$, the following first-order stationary condition holds for the training algorithm Hier-Local-QSGD: 
	
	\begin{align}
		\frac{1}{K} \sum_{k=0}^{K-1} \mathbb{E} \left \Vert \nabla f(x_k) \right \Vert ^2 & \leq  \frac{2(f(x_0) - f^*)}{\eta K\tau_1\tau_2} \notag  \\
		& + \frac{L^2\eta^2}{2}\left[ \frac{(1+q_1)}{n/s}  \tau_1 (\tau_2-1) 
		 + (\tau_1-1)\right] \sigma^2 \notag \\ &+L\eta\frac{1}{n}(1+q_1)(1+q_2)\sigma^2 .  \label{eq:27}
	\end{align}
	
\end{theorem}

\begin{remark} \label{remark1}
	The bound in \eqref{eq:27} can be simplified for specific settings. Specifically, by letting the step size $\eta = \frac{1}{L\sqrt{K\tau_1\tau_2}}$, we have the following convergence rate:
	\small
	\begin{align}
	\frac{1}{K} \sum_{k=0}^{K-1} \mathbb{E} \left \Vert \nabla f(x_k) \right \Vert^2 
	& \leq  \frac{2L(f(x_0) - f^*)}{\sqrt{K\tau_1\tau_2}}  \notag \\
	& + \frac{1}{K\tau_1\tau_2} \frac{1}{2}\left[ \frac{(1+q_1)}{n/s}  \tau_1 (\tau_2-1)+ (\tau_1-1)\right] \sigma^2 \notag \\ & +\frac{1}{\sqrt{K\tau_1\tau_2}}\frac{(1+q_1)(1+q_2)\sigma^2}{n}
	\end{align}
	\normalsize
	This means that the algorithm can achieve an overall convergence rate of $\mathcal{O}(\frac{1}{\sqrt{K\tau_1\tau_2}})$ when the learning rate is sufficiently small. This guarantees that Hier-Local-QSGD can greatly improve the communication efficiency while achieving comparable performance as the baseline algorithm without partial edge aggregation and model quantization for non-convex loss functions \cite{reisizadeh2020fedpaq}.
\end{remark}

\begin{remark} \label{remark2}
	When the condition $G \geq 0$ is satisfied, the optimal parameters to achieve the fastest convergence speed in terms of local update iterations are: $\tau_1 = \tau_2 = 1,\text{ and } q_1 = q_2 = 0$. In this special case, Hier-Local-QSGD degrades to the conventional SGD. Note that this does not mean that the convergence will be the fastest in terms of wall clock time, as the communication latency is different for the edge side update and the cloud side update. Furthermore, the local computation also takes time. The optimal parameters selection depends on the ratio of $\tau_1\text{ and }\tau_2$ and we shall derive an adaptive algorithm to determine the parameters $\tau_1\text{ and } \tau_2$ in Section \ref{Applications}.
\end{remark}

\begin{remark} \label{remark3}
	When $\tau_2 =1, q_1=q_2=0$, which means that there is no partial aggregation nor quantization, we recover the result of \cite{wang2018cooperative}. It is noted that our result does not coincide exactly with the result in \cite{reisizadeh2020fedpaq} for the two-layer \textit{FedPAQ} algorithm when we set $\tau_2=1$, i.e., FedAvg with quantization. This is because the expected gradient norm on the left hand side of \eqref{eq:27} is different. An average of the expected gradient norm for the model parameters after every $\tau_1\tau_2$ updates, i.e. $\{x_k \}_{k=0,\dots,K-1}$, is considered in this paper, while an average of the expected gradient norm for the
	auxiliary virtual model parameters at every update step, i.e., $\{\bar{x}_{k,t} \}_{k=0,\dots,K-1, t=0,\dots, \tau}$, is considered in \cite{reisizadeh2020fedpaq}.
\end{remark}

\begin{remark} \label{remark5}
	One implication of the bound \eqref{eq:27} is that the fewer the number of edge servers, i.e., $s$, the faster the convergence. When the number of participated clients in the FL system is fixed, the partial edge aggregation will incorporate more clients if there are fewer edge servers available in the system. The variance caused by the partial aggregation decreases in this case, and hence the convergence will be faster.
\end{remark}
\begin{remark} \label{remark4}
	For the locally estimated gradient $\tilde{ \nabla} f$, a batch of data of size $b$ can also be used. In this case, the only difference in the analysis is that the variance of the SGD in Assumption \ref{assumption2} decreases from $\sigma^2$ to $\sigma^2 / b$.
\end{remark}

\begin{remark} \label{remark7}
	(Extension to the non-IID data) As stated in Assumption \ref{assumption2}, the locally estimated stochastic gradient is assumed to be an unbiased estimate of the true gradient of the loss function $f(x)$. In FL, this independent and identical data (IID) assumption may not be satisfied. The extension to the non-IID case is non-trivial since the locally estimated gradient direction diversifies among different clients when they are performing local updates, and thus the aggregated gradient descent direction may not be the correct one. By replacing the weak assumption that gradient variance is upper bounded with a strict assumption that the gradient norm is bounded, a convergence result could be derived but the variance term is expected to be quadratic with respect to the local update steps. In this paper, we aim to derive a tighter bound which can be further utilized for system design, and thus, we have the IID assumption in Assumption \ref{assumption2}. 
	
	{
	Relaxing the uniform variance bound $\mathbb{E} [\left \Vert \tilde{ \nabla} f_i(x) - \nabla f(x) \right \Vert^2 |x ]\leq \sigma^2$ to different variance bounds $\sigma_i^2$ in Assumption \ref{assumption2} is straightforward. Our proof can be easily extended to this case and get the following result
	\small
	\begin{align*}
	\frac{1}{K} \sum_{k=0}^{K-1} \mathbb{E} \left \Vert \nabla f(x_k) \right \Vert^2 
	\leq  & \frac{2(f(x_0) - f^*)}{\eta K\tau_1\tau_2} \\
	&+ \frac{L^2\eta^2}{2}\left[ (\tau_1-1)\sigma_c^2 + (1+q_1) \tau_1 (\tau_2-1) \sigma_e^2\right] \\
	& +L\eta\frac{1}{n}(1+q_1)(1+q_2)\sigma_c^2 , \label{eq:34}
	\end{align*}
	\normalsize
	where $\sigma_c ^2 = \frac{1}{n}\sum_{i=1}^{n}\sigma_i^2$, $\sigma_e^2 = \frac{1}{n} \sum_{\ell=1}^{s}[ \frac{1}{m^\ell} \sum_{j\in \mathcal{C}^\ell}\sigma_j^2]$. Unfortunately such extension cannot handle the non-IID data.
	} 
	In Section \ref{Simulations}, we will perform experiments on non-IID data and will show that the adaptive interval selection scheme proposed based on the analytical results works well under the non-IID case.
\end{remark}

\subsection{Examples of System Design Guidelines}
This part provides two examples to illustrate the system design guidelines that can be obtained from our analysis.
\subsubsection{Too much quantization suggests infrequent communication}\label{quantization}
By observing Eq. \eqref{eq:27}, we can rewrite the variance term on the right hand side as follows:

\small
\begin{equation}
	\frac{L^2\eta^2}{2}\left[ \frac{(1+q_1)}{n/s}  \tau_1 \tau_2+ (1-\frac{1+q_1}{n/s} )\tau_1 -1 \right] \sigma^2  +L\eta\frac{(1+q_1)(1+q_2)\sigma^2}{n}.
	\label{eq:32}
\end{equation}
\normalsize
Suppose that $\tau_1\tau_2, q_1 \text{ and } q_2$ are fixed. Then, when $q_1 < n/s -1$, i.e., the quantization output is not too far from the input, then a smaller $\tau_1$ leads to a lower upperbound. That is, frequent local aggregation leads to a faster convergence. On the other hand, when $q_1 > n/s -1$, i.e., the quantization output is very inaccurate, then a larger $\tau_1$ leads to a lower upperbound. That is, infrequent local aggregation leads to a faster convergence rate.  
\par
This is a very interesting observation since in the existing analyses for FedAvg \cite{khaled2020tighter,reisizadeh2020fedpaq}, the local aggregation interval has always been positively correlated with the additional variance term. This means a smaller aggregation interval always leads to a faster convergence rate. On the other hand, our result suggests that when the communication frequency with the cloud server is fixed, how the edge-client aggregation interval influences the convergence rate depends on the edge-client quantization level parameter $q_1$. If there is too much quantization, then less frequent communication between the edge and the client is preferred. This is because when the accumulated error caused by the quantization on the communication between the clients and edge is higher than that caused by the multi-step local update, less communication will be better. In practice, when the communication bandwidth is limited and the quantization must be high, our analytical result suggests that infrequent communication between the client and edge servers is preferred. We will verify this flipped phenomenon of $\tau_1$ in Section \ref{Simulations}.

\subsubsection{Edge-client association strategy has no impact on the convergence}\label{association}

The edge-client association is a unique resource allocation problem in Hierarchical Federated Edge Learning.
This problem has been discussed in \cite{luo2020hfel} with a rough convergence analysis framework to capture the learning performance and to formulate a joint learning and resource allocation problem to accelerate training and save energy. Intuitively, it will be beneficial to the overall learning performance to have many devices connected to each edge server. However, for a given edge server with limited spectrum resources, when more clients are connected to the server, less bandwidth will be assigned to each client, which results in a longer communication delay. Thus, the learning performance and spectrum resource allocation are intertwined with each other, which makes the optimization problem different from the general computation offloading problem in the conventional MEC framework \cite{7541539}.

\par
In this paper, by analyzing the convergence of the Hierarchical FL with the Hier-Local-QSGD algorithm, we find that, when the number of edge servers is fixed and each edge server is associated with at least one client, the convergence speed with respect to (w.r.t.) the iterations is irrelevant to the client-edge association strategy. Thus, to accelerate training of the overall system, we only need to minimize the communication delay for each aggregation. Due to the synchronization requirement in the Hier-Local-QSGD training algorithm at each aggregation step, the delay for each aggregation step is determined by the slowest client.

\par
It must be noted that a recent work \cite{wang2020local} also analyzed the hierarchical local SGD algorithm for general non-convex loss functions. In their convergence result, it was found that the client-distribution will influence the final error bound in contrast to what we observed. This can be explained by the fact that different weight coefficients are used when averaging the updates from the edge to the cloud. In our Hier-Local-QSGD algorithm, the weighted coefficient $m^\ell / n$ is used while in the HF-SGD algorithm in \cite{wang2020local}, the uniform coefficient $1/s$ is used. When performing partial edge aggregation, the additional variance introduced by partial aggregation is inversely proportional to the number of clients to be aggregated, i.e., $m^\ell$. Thus, adopting a weighted average policy at the cloud aggregation step will balance the additional variance that is introduced by the edge server with fewer clients while a uniform average policy fails to do so. We also verify this observation through simulations in Section \ref{Simulations}.

\subsection{Proof Outline}
We now give an outline of the proof for Theorem \ref{theorem1}. Detailed proofs of the lemmas are deferred to Appendix \ref{Appendix}.
\par
To assist the analysis, a virtual auxiliary variable $ \bar{x}_k$ is introduced, which is the average of the unquantized updates from the edge servers and defined as follows:
\begin{equation}
\bar{x }_{k+1}= x_k + \sum_{\ell=1}^s  \frac{m^\ell}{n} (u_{k,\tau_2}^\ell - u_{k,0}^\ell). \label{eq:8}
\end{equation}

The evolution of the true and auxiliary model parameters $x_{k,t_2,t_1}^i$, $ \bar{x }_{k+1}$, and $x_{k+1}$ is specified as follows:
\small
\begin{align}
& x^i_{k,t_2,t_1} \notag \\
& = x_k - \eta \sum_{\beta=0}^{t_1-1} \tilde{\nabla}f_i(x^i_{k,t_2,\beta}) - \eta \sum_{\alpha=0}^{t_2-1}\sum_{j\in \mathcal{C}^{\ell _{i}}} \frac{1}{m^{\ell_i}} Q_1^{(\alpha)}\left [ \sum_{\beta=0}^{\tau_1-1} \tilde{\nabla}f_j(x^j_{k,\alpha,\beta}) \right] \label{eq:9}
\end{align}
\vspace{-1mm}
\begin{equation}
\bar{x}_{k+1} = x_k - \eta \sum_{\ell \in [s]} \frac{m^\ell}{n} 
\frac{1}{m^\ell}\sum_{\alpha=0}^{\tau_2-1}\sum_{j\in \mathcal{C}^{\ell}} Q_1^{(\alpha)}\left [ \sum_{\beta=0}^{\tau_1-1} \tilde{\nabla}f_j(x^j_{k,\alpha,\beta}) \right ]  \label{eq:10}
\end{equation}

\begin{equation}
x_{k+1} = x_k - \eta \sum_{\ell \in [s]} \frac{m^\ell}{n} Q_2\left \{
\frac{1}{m^\ell}\sum_{\alpha=0}^{\tau_2-1}\sum_{j\in \mathcal{C}^{\ell}} Q_1^{(\alpha)}\left [ \sum_{\beta=0}^{\tau_1-1} \tilde{\nabla}f_j(x^j_{k,\alpha,\beta}) \right] \right\}.  \label{eq:11}
\end{equation}
\normalsize
\par
The proof proceeds as follows: Using the property of $L$-smooth functions, we first prove a bound in Lemma \ref{lemma1} of the evolution process of the cloud model parameter $\left\{ x_k \right\}$, which depends on three terms, i.e. $\mathbb{E} \left \langle \nabla f(x_k), \bar{x}_{k+1} - x_k \right \rangle, \mathbb{E} \left \Vert \bar{x}_{k+1} - x_k \right \Vert ^2, \text{and } \mathbb{E} \left \Vert x_{k+1} -\bar{x}_{k+1} \right \Vert^2$. In Lemmas \ref{lemma2}, \ref{lemma3}, and \ref{lemma4}, we derive upper bounds of the three terms, respectively, and characterize their relationships to the aggregation parameters $\tau_1,\text{ and } \tau_2$ along with the quantization variance parameters $q_1, \text{ and }q_2$.
\begin{lemma} [One round of global aggregation]
	\label{lemma1}
	With Assumptions \ref{assumption1} and \ref{assumption2}, we have the following relationship between $x_{k+1}$ and $x_k$:
	\begin{align}
	\mathbb{E} f(x_{k+1}) -\mathbb{E} f(x_{k}) &\leq  \mathbb{E} \left \langle \nabla f(x_k), \bar{x}_{k+1} - x_k \right \rangle \notag \\
	&+ \frac{L}{2} \mathbb{E} \left \Vert \bar{x}_{k+1} - x_k \right \Vert ^2+ \frac{L}{2} \mathbb{E} \left \Vert x_{k+1} -\bar{x}_{k+1} \right \Vert^2 \label{eq:12}
	\end{align}
\end{lemma}

Lemma \ref{lemma1} follows from the property of the $L-$smoothness in Assumption \ref{assumption1}. We next bound the three terms on the right hand side of Eqn. \eqref{eq:12}. 
\begin{lemma} \label{lemma2}
	With Assumptions \ref{assumption1}, \ref{assumption2} and \ref{assumption3}, $\mathbb{E} \langle \nabla f(x_k), \bar{x}_{k+1} - x_k \rangle$  is bounded as follows:
	\begin{align*}
	&\mathbb{E} \langle \nabla f(x_k), \bar{x}_{k+1} - x_k \rangle \\
	\leq
	&-\frac{\eta D}{2} \times \frac{1}{n} \sum_{i=1}^n \sum_{\alpha=0}^{\tau_2-1}\sum_{\beta=0}^{\tau_1-1} \mathbb{E} \left\Vert \nabla f(x_{k,\alpha, \beta}^i) \right\Vert^2 \\
	&+ \frac{\tau_1 \tau_2}{2} \left[ (\tau_1-1) + \frac{s}{n}(1+q_1) \tau_1 (\tau_2-1) \right] \sigma^2
	\end{align*}
	, where $D= \left\{ 1-L^2\eta^2\left[\frac{\tau_1(\tau_1-1)}{2}+\tau_1\tau_2\left(\frac{\tau_2(\tau_2-1)}{2}+q_1\tau_2 \right) \right] \right\}$.
\end{lemma}

\begin{lemma} \label{lemma3}
	With Assumptions \ref{assumption1}, \ref{assumption2} and \ref{assumption3}, $\mathbb{E} \left \Vert \bar{x}_{k+1} - x_k \right \Vert ^2$ is bounded as follows:
	\begin{align}
	&\mathbb{E} \left \Vert \bar{x}_{k+1} - x_k \right \Vert ^2 \notag \\
	&\leq \eta^2 \left( \tau_1\tau_2 +\frac{q_1\tau_1}{n} \right)\frac{1}{n}\sum_{i=1}^n \sum_{\alpha=0}^{\tau_2-1}\sum_{\beta=0}^{\tau_1-1} \mathbb{E} \left\Vert \nabla f(x_{k,\alpha, \beta}^i) \right\Vert^2 \notag \\
	&+ \eta^2\frac{1}{n}(1+q_1)\tau_1\tau_2\sigma^2
	\end{align}
\end{lemma}

\begin{lemma} \label{lemma4}
	With Assumptions \ref{assumption1}, \ref{assumption2} and \ref{assumption3}, $E \left \Vert x_{k+1} -\bar{x}_{k+1} \right \Vert^2$ is bounded as follows:
	\begin{align}
	&E \left \Vert x_{k+1} -\bar{x}_{k+1} \right \Vert^2 \notag \\
	&\leq \eta^2 q_2 \left( \tau_1\tau_2 +\frac{q_1\tau_1}{n} \right)\frac{1}{n}\sum_{i=1}^n \sum_{\alpha=0}^{\tau_2-1}\sum_{\beta=0}^{\tau_1-1} \mathbb{E} \left\Vert \nabla f(x_{k,\alpha, \beta}^i) \right\Vert^2 \notag \\
	&+ \eta^2\frac{1}{n}(1+q_1)q_2\tau_1\tau_2\sigma^2
	\end{align}
\end{lemma}

By combining Lemmas \ref{lemma1} to \ref{lemma4}, we now have the following:
\small
\begin{align}
& \mathbb{E} f(x_{k+1}) -\mathbb{E} f(x_{k}) \leq - \frac{\eta}{2} \tau_1\tau_2 \mathbb{E} \left \Vert \nabla f(x_k) \right \Vert^2 \notag \\
& -\frac{\eta D}{2} \times \frac{1}{n}\sum_{i=1}^n \sum_{\alpha=0}^{\tau_2-1}\sum_{\beta=0}^{\tau_1-1} \mathbb{E} \left\Vert \nabla f(x_{k,\alpha, \beta}^i) \right\Vert^2 \notag \\
& + \frac{L^2\eta^3}{4} \tau_1 \tau_2\left[ (\tau_1-1) + \frac{s}{n}(1+q_1) \tau_1 (\tau_2-1) \right] \sigma^2 \notag \\ &+\frac{L\eta^2}{2}\frac{1}{n}(1+q_1)(1+q_2)\tau_1\tau_2\sigma^2.
\end{align}
\normalsize

For a sufficiently small $\eta$, and when the following condition is satisfied:
\small
\begin{equation}
D- L\eta(1+q_2) \left( \tau_1\tau_2 +\frac{q_1\tau_1}{n} \right) \geq 0,
\end{equation}
\normalsize
we have
\small
\begin{align}
\mathbb{E} f(x_{k+1}) -\mathbb{E} f(x_{k})& \leq - \frac{\eta}{2} \tau_1\tau_2 \mathbb{E} \left \Vert \nabla f(x_k) \right \Vert^2 \notag \\ 
+ & \frac{L^2\eta^3}{4} \tau_1 \tau_2\left[ (\tau_1-1) + \frac{s}{n}(1+q_1) \tau_1 (\tau_2-1) \right] \sigma^2 \notag \\
 +&\frac{L\eta^2}{2}\frac{1}{n}(1+q_1)(1+q_2)\tau_1\tau_2\sigma^2. \label{eq:24}
\end{align}
\normalsize
By summing \eqref{eq:24} and re-arranging the terms, we obtain the main result in Theorem \ref{theorem1}.
\par
Now, we have derived the convergence result for the proposed Hier-Local-QSGD algorithm w.r.t. the update iterations, i.e., $k$. Next, by applying the theoretical analysis to a system design problem, i.e., the aggregation interval selection problem, we will illustrate how the analytical result can be used to reduce the overall training latency in hierarchical federated learning.

\begin{table*}
	\centering
	\begin{tabular}{|c | c |c | c| c| c | c |c | c| c|} 
		\hline
		\scriptsize{Sparsification} & \multicolumn{3}{|c|}{$q_1=0$} & \multicolumn{3}{|c|}{$q_1=19$} & \multicolumn{3}{|c|}{$q_1=65.57$} \\ 
		\hline
		\text{$\tau_1\times\tau_2$ } & 125*2 & 50*5 & 10*25& 125*2 & 50*5 & 10*25 &125*2 & 50*5 & 10*25 \\ 
		\hline
		$\alpha=100$ & 0.8795 & 0.8870 & 0.8930 & 0.8643 & 0.8722 & 0.8930 &0.8010 &0.7772& 0.6810 \\ 
		\hline
		$\alpha=1$ & 0.8721 & 0.8772 & 0.8882 & 0.8583  &0.8593 & 0.8613 &0.8025 & 0.7703 & 0.6648 \\ 
		\hline
		$\alpha=0.1$ & 0.8173 & 0.8256 & 0.8453 & 0.7979 & 0.8000& 0.8099 &0.7042 & NaN & NaN \\ 
		\hline
	\end{tabular}
	\vspace{5pt}
	\caption{{Accuracy w.r.t. communication round for different values of $\tau_1$ with $q_1 = 0$, $q_1 = 19$  and $q_1 = 65.67$. $n=20,s=4, \tau_1\tau_2=250,\text{ and } q_2=0$, on a standard \textit{CIFAR-10} dataset.}	}
	\label{table2}
	\vspace{20pt}
	\centering
	\begin{tabular}{|c | c |c | c| c| c | c |c | c| c|} 
		\hline
		\scriptsize{Rounding} & \multicolumn{3}{|c|}{$8$bits} & \multicolumn{3}{|c|}{$4$ bits} & \multicolumn{3}{|c|}{$2$ bits} \\ 
		\hline
		\text{$\tau_1\times\tau_2$ } & 125*2 & 50*5 & 10*25& 125*2 & 50*5 & 10*25 &125*2 & 50*5 & 10*25 \\ 
		\hline
		$\alpha=100$ & 0.8792 & 0.8866 & 0.8927 & 0.8339 & 0.8248 & 0.7827 & 0.7333 &0.6973& 0.6092\\ 
		\hline
	\end{tabular}
	\vspace{5pt}
	\caption{{{Accuracy w.r.t. communication round for different values of $\tau_1$ with quantization levels as $b_1 = 8$bits, $b_1 = 4$bits  and $b_1 = 2$ bits. $n=20,s=4, \tau_1\tau_2=250,\text{ and } q_2=0$, IID split ($\alpha = 100$) on a standard \textit{CIFAR-10} dataset.}}}
	\label{table3}
\end{table*}

\begin{figure*}[t]
	\centering
	\begin{subfigure}{.25\textwidth}
		\captionsetup{width=0.9\textwidth}
		\centering
		\includegraphics[width=0.9\linewidth]{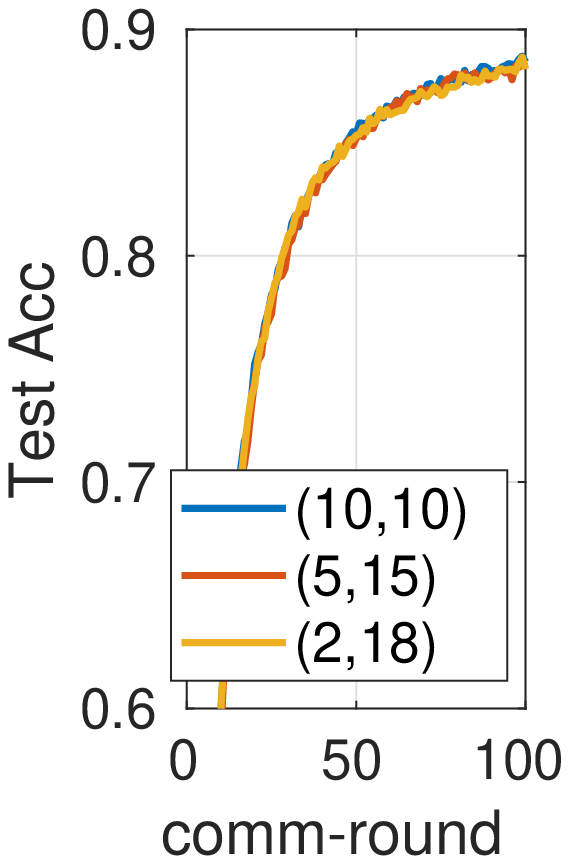}
		\caption{weighted average, $\alpha=100$.}
		\label{figsim1:sub1}
	\end{subfigure}
	\begin{subfigure}{.24\textwidth}
		\captionsetup{width=0.9\textwidth}
		\centering
		\includegraphics[width=0.9\linewidth]{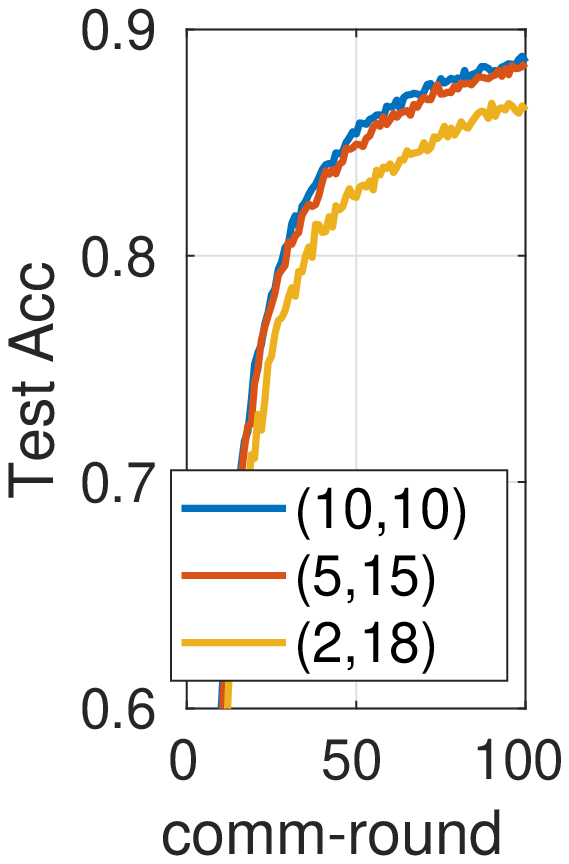}
		\caption{uniform average \cite{luo2020hfel},  $\alpha=100$.}
		\label{figsim1:sub2}
	\end{subfigure}%
	\begin{subfigure}{.24\textwidth}
		\captionsetup{width=0.9\textwidth}
		\centering
		\includegraphics[width=0.9\linewidth]{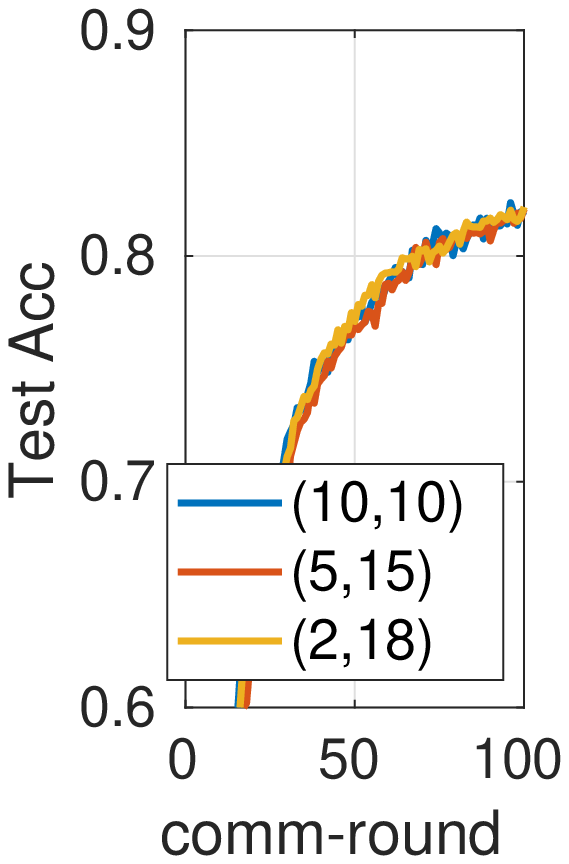}
		\caption{weighted average, $\alpha=0.1$.}
		\label{figsim1:sub3}
	\end{subfigure}
	\begin{subfigure}{.24\textwidth}
		\captionsetup{width=0.9\textwidth}
		\centering
		\includegraphics[width=0.9\linewidth]{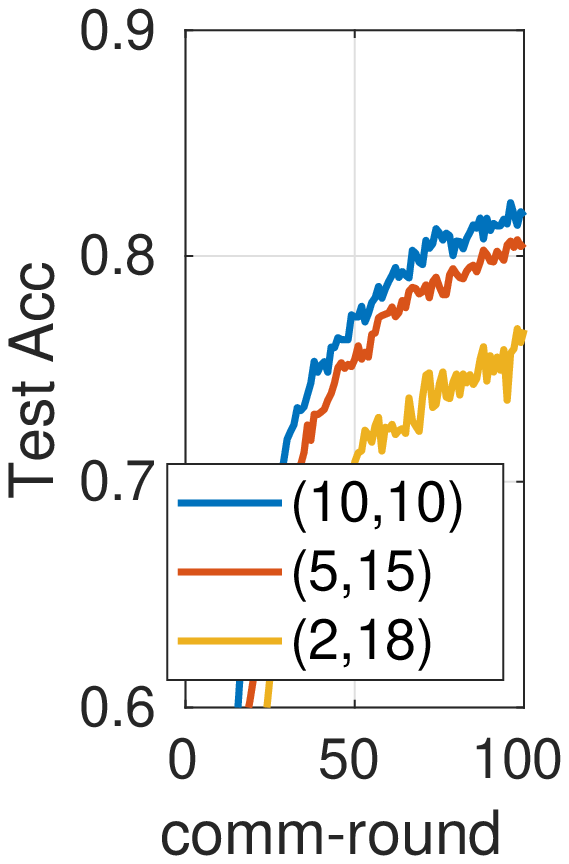}
		\caption{uniform average \cite{luo2020hfel}, $\alpha=0.1$.}
		\label{figsim1:sub4}
	\end{subfigure}%
	\vspace{20pt}
	\caption{{Accuracy with different edge-client association strategies using weighted average and uniform average. $\tau_1= 50,\tau_2 = 5, q_2 =q_1 = 0, n=20,\text{ and } s = 2$, on a standard \textit{CIFAR-10} dataset.}}
	\label{fig:sim1}
\end{figure*}

\section{Adaptive Aggregation interval Control} \label{Applications}
In this section, we illustrate the application of the convergence analysis to investigate the aggregation interval selection problem, i.e., how to optimize the system parameters $\tau_1\text{ and } \tau_2$. To focus on the aggregation interval control, we assume that the two quantization variance parameters $q_1\text{ and } q_2$ are fixed throughout the training process.


\par
To characterize the trade-off between the learning performance and the communication efficiency, we assume that the clients are with the same computation and communication resources. Further assume that the local computation time for one SGD iteration is $D_{comp}$, the communication delay of transmitting a quantized model updates between the client (device) and edge is $D_{de}$, and the communication delay of transmitting a quantized model between the edge and cloud is $D_{ec}$. Then, the wall clock time $T$ of $K$ rounds of cloud-aggregation is given by
$
	T = K \left( \tau_1\tau_2 D_{comp} +\tau_2 D_{de}  + D_{ec} \right).
$
By substituting $T$ into \eqref{eq:27}, the minimal expected squared gradient norm is bounded by:
\small
\begin{align}
&\frac{2(f(x_0) - f^*)}{T} \left( D_{comp} + \frac{D_{de}}{\tau_1} + \frac{D_{ec}}{\tau_1\tau_2}\right) \notag \\
+& \frac{L^2\eta^2}{2}\left[ \frac{1+q_1}{n/s}  \tau_1 (\tau_2-1)+ (\tau_1-1)\right] \sigma^2  +L\eta\frac{1}{n}(1+q_1)(1+q_2)\sigma^2  \label{eq:29}
\end{align}
\normalsize
From the bound in \eqref{eq:29}, we can clearly see the accuracy-latency trade-off when choosing different values of $\tau_1,\text{ and } \tau_2$. 
For a given setting and a specific performance requirement, e.g., a deadline of training, we can determine the values of these key parameters accordingly through minimizing \eqref{eq:29} by setting the derivatives w.r.t. the corresponding parameters to zero.

\begin{theorem}
	For Hier-Local-QSGD, with the same assumptions as Theorem \ref{theorem1}, the error bound in \eqref{eq:29}  is minimized when we select the two aggregation intervals as:
	\begin{equation}
		\tau_1^* = \sqrt{\frac{4(f(x_0)-f^*)D_{de}}{\eta^3L^2\sigma^2T(1-\frac{1+q_1}{n/s})}}, \quad
		\tau_2^* =  \sqrt{\frac{D_{ec}}{D_{de}}\frac{(1-\frac{1+q_1}{n/s})}{\frac{1+q_1}{n/s}}}
		\label{eq:31}
	\end{equation}
when $1+q_1<n/s$.
\end{theorem}

\begin{proof}
	Denoting $\tau_1\tau_2 = \tau$, then it can be easily proved that \eqref{eq:29} is convex w.r.t $\tau_1 \text{ and }\tau$ when $1+q_1<n/s$. Then, by setting the partial derivatives w.r.t $\tau_1 \text{ and } \tau$ to 0,
	we can get the solution in Eq. \eqref{eq:31}.
\end{proof}

\par
{By adopting a similar idea from \cite{MLSYS2019Wang}}, in the adaptive interval selection scheme, the whole training procedure is split into uniform wall clock time intervals with the same wall clock time length $T_0$. At the beginning of each time interval, we will use \eqref{eq:29} to estimate the best aggregation interval for the next $T_0$ time period. 
{
	The optimal local update step value, i.e., $(\tau_1^j)^*$ in the $j$-th adapt wall clock time interval, i.e., $t\in \left( (j-1)T_0, jT_0 \right)$ is then $\sqrt{\frac{4(f(x_{t=(j-1)T_0})-f^*)D_{de}}{\eta^3L^2\sigma^2T(1-\frac{1+q_1}{n/s})}}$. The values of the Lipschitz constant $L$, stochastic gradient variance $\sigma^2$ and the lower bound of the non-negative loss $f^*$ are unknown. However, by approximating $f^*$ by $0$, the rest unknown parameters can be canceled with division. Such approximation is also adopted in \cite{fasttcom20,MLSYS2019Wang} . By dividing  $(\tau_1^j)^*$ by $(\tau_1^0)^*$,
}
we can have the following update rule for $t\in \left( (j-1)T_0, jT_0 \right)$:
\begin{equation}
	\tau_1^j =  \left\lceil \sqrt{\frac{f(x_{t=(j-1)T_0})}{f(x_{t=0})}}\tau_1^0 \right\rceil, \quad
	\tau_2^j= \left\lceil \sqrt{\frac{D_{ec}}{D_{de}}\frac{(1-\frac{1+q_1}{n/s})}{\frac{1+q_1}{n/s}}}  \right\rceil
	\label{eq:30}
\end{equation}
where $\tau_1^0$ is the predefined value of $\tau_1$ for the first time interval $t\in \left( 0, T_0 \right)$. 
{The computation of $\tau_1$ only requires the value of the training loss. The adaptive procedure is monitored by the server, which collects the local training loss from each client and computes the update of $\tau_1$ and then sends it back to clients. The additional communication, i.e., uploading and downloading one scalar value between the clients and server, is negligible compared to the cost of the model parameters.
}

\begin{figure*}[t]
	\centering
	\begin{subfigure}{.33\textwidth}
		\captionsetup{width=0.8\textwidth}
		\centering
		\includegraphics[width=.9\linewidth]{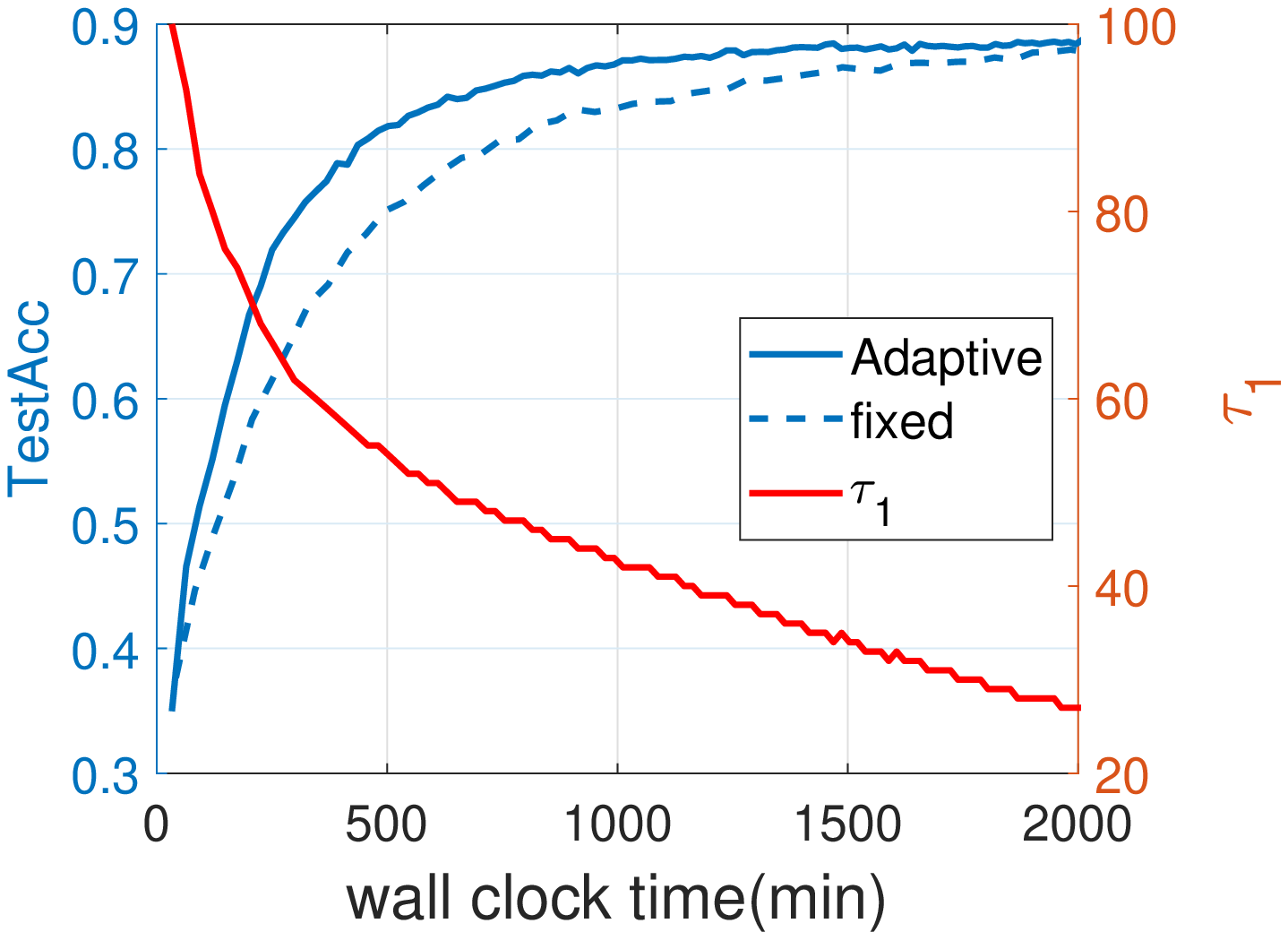}
		\caption{\textit{CIFAR-10}, $\alpha=100$}
		\vspace{20pt}
		\label{figsim3:sub1}
	\end{subfigure}%
	\begin{subfigure}{.33\textwidth}
		\captionsetup{width=0.8\textwidth}
		\centering
		\includegraphics[width=.9\linewidth]{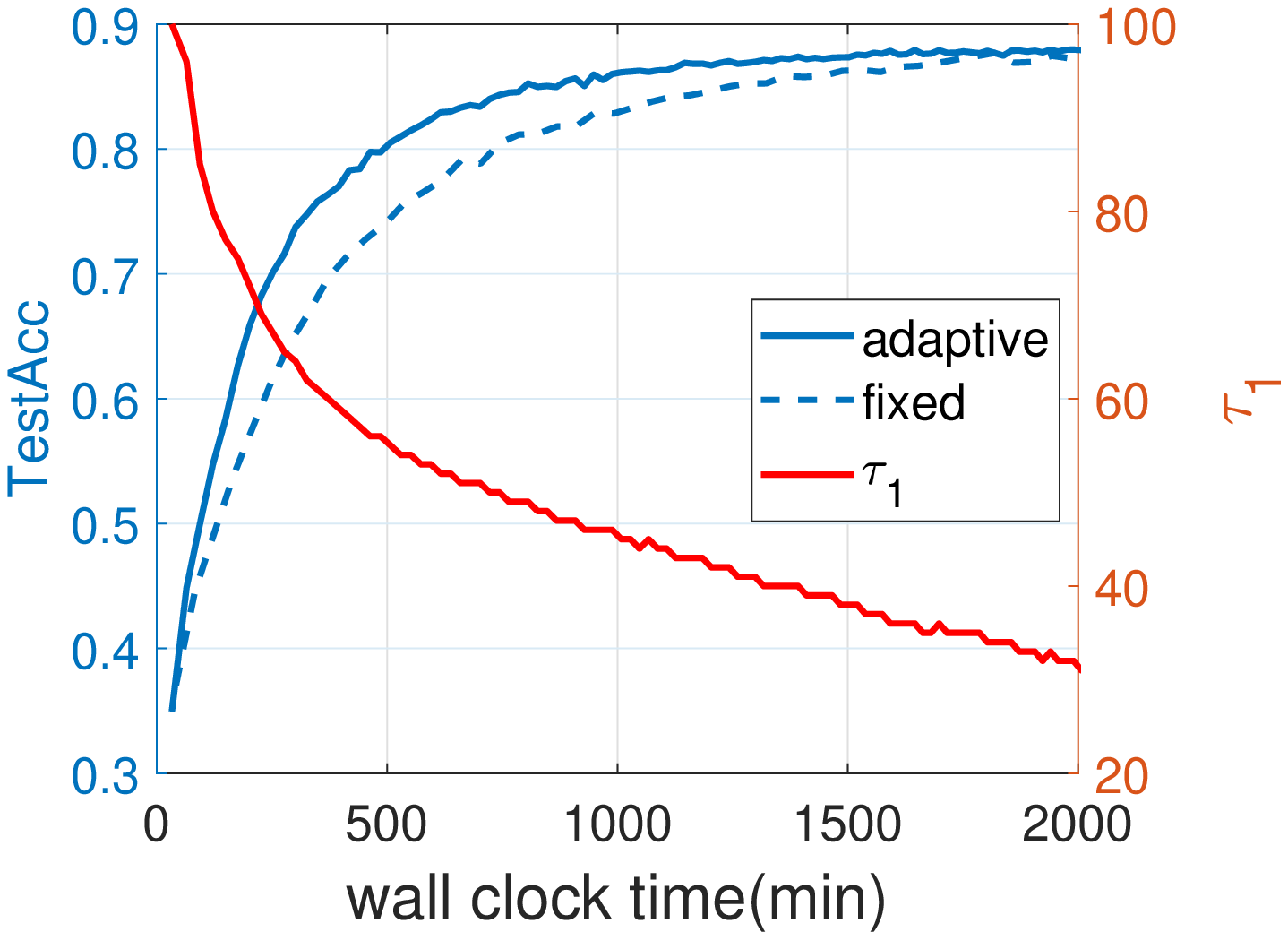}
		\caption{\textit{CIFAR-10}, $\alpha=1$}
		\vspace{20pt}
		\label{figsim3:sub2}
	\end{subfigure}
	\begin{subfigure}{.33\textwidth}
		\captionsetup{width=0.8\textwidth}
		\centering
		\includegraphics[width=.9\linewidth]{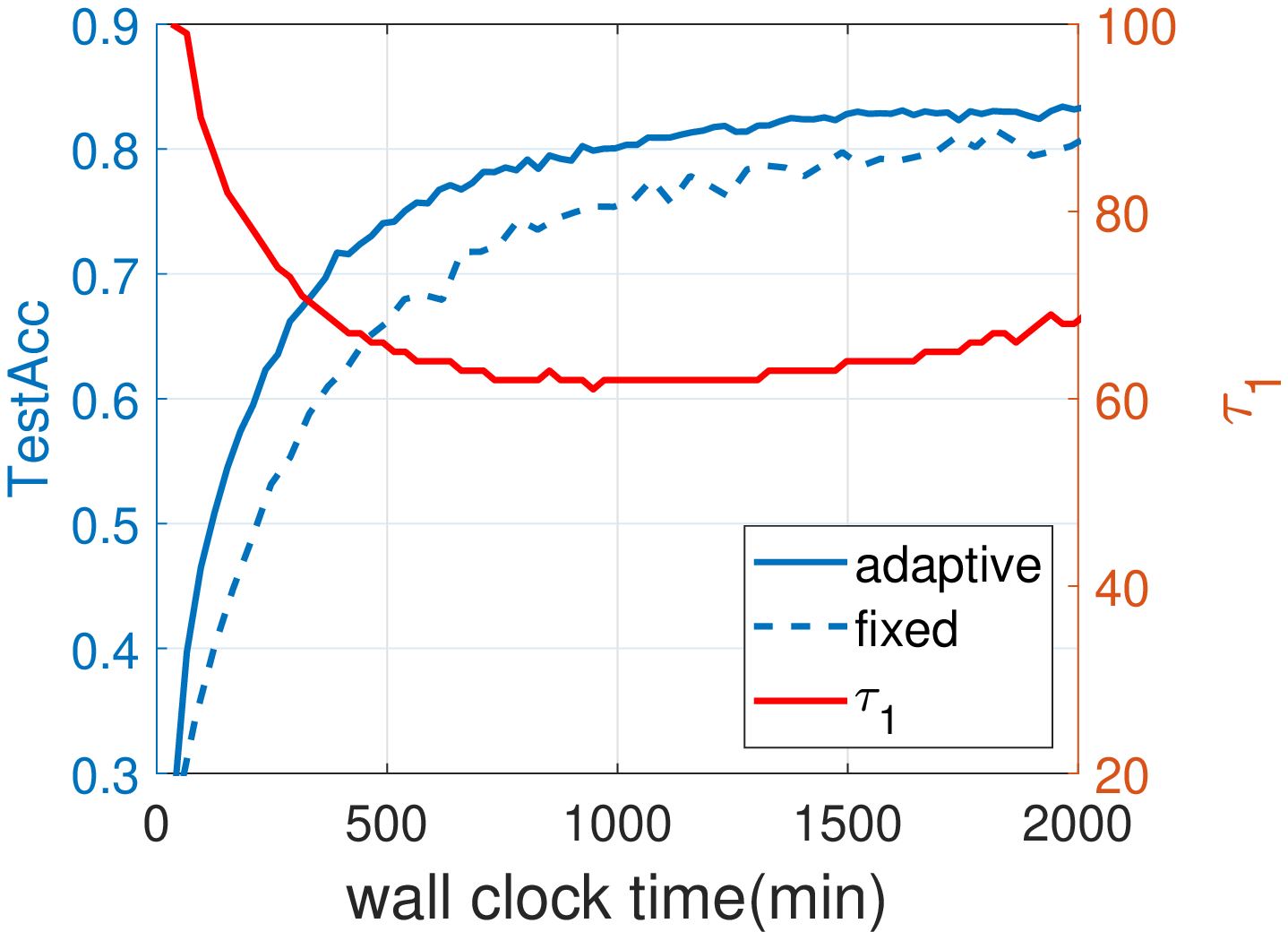}
		\caption{\textit{CIFAR-10}, $\alpha=0.1$}
		\vspace{20pt}
		\label{figsim3:sub3}
	\end{subfigure}
	\begin{subfigure}{.33\textwidth}
		\captionsetup{width=0.8\textwidth}
		\centering
		\includegraphics[width=.9\linewidth]{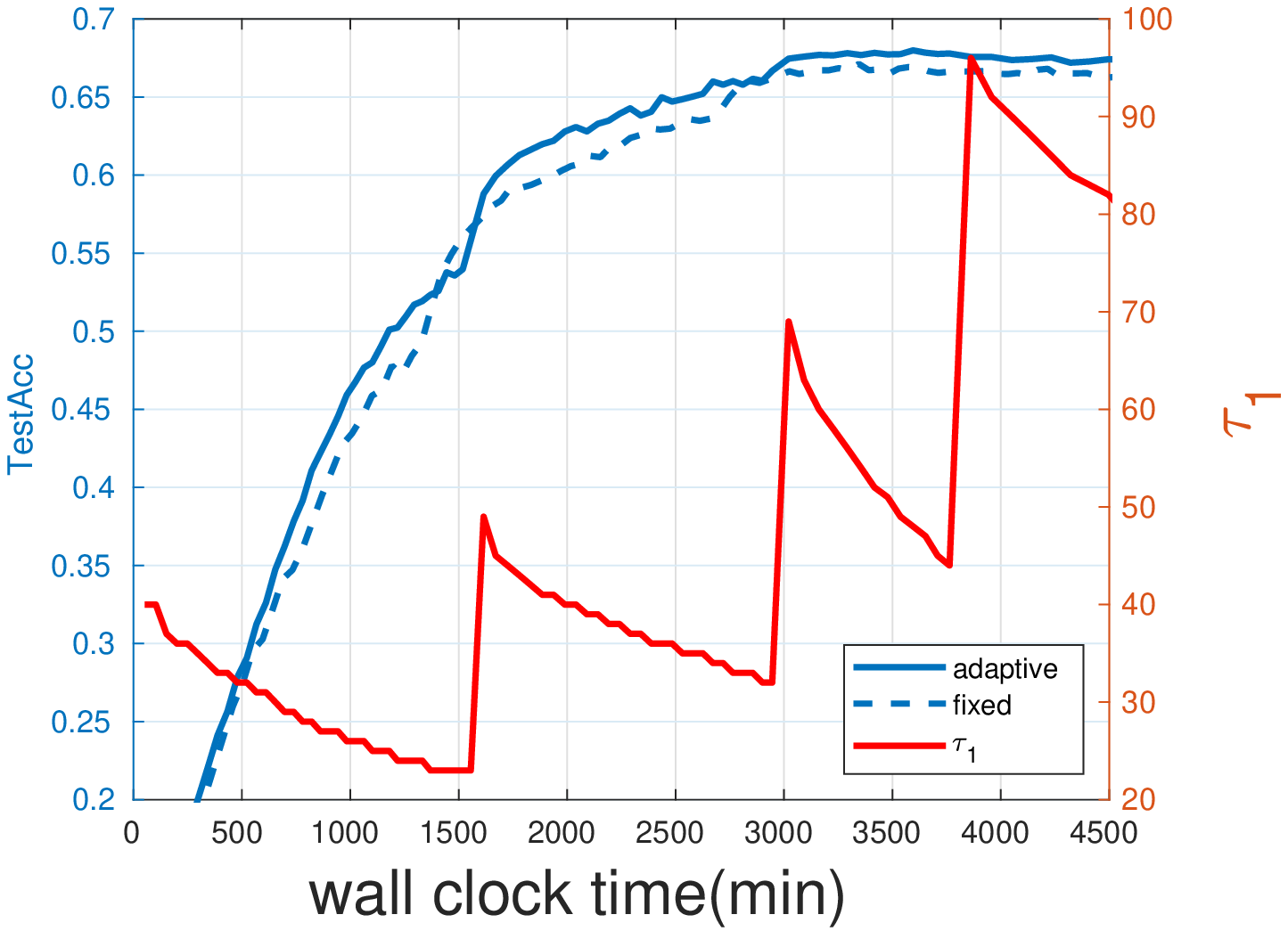}
		\caption{\textit{CIFAR-100}, $\alpha=100$.}
		\vspace{20pt}
		\label{figsim4:sub1}
	\end{subfigure}%
	\begin{subfigure}{.33\textwidth}
		\captionsetup{width=0.8\textwidth}
		\centering
		\includegraphics[width=.9\linewidth]{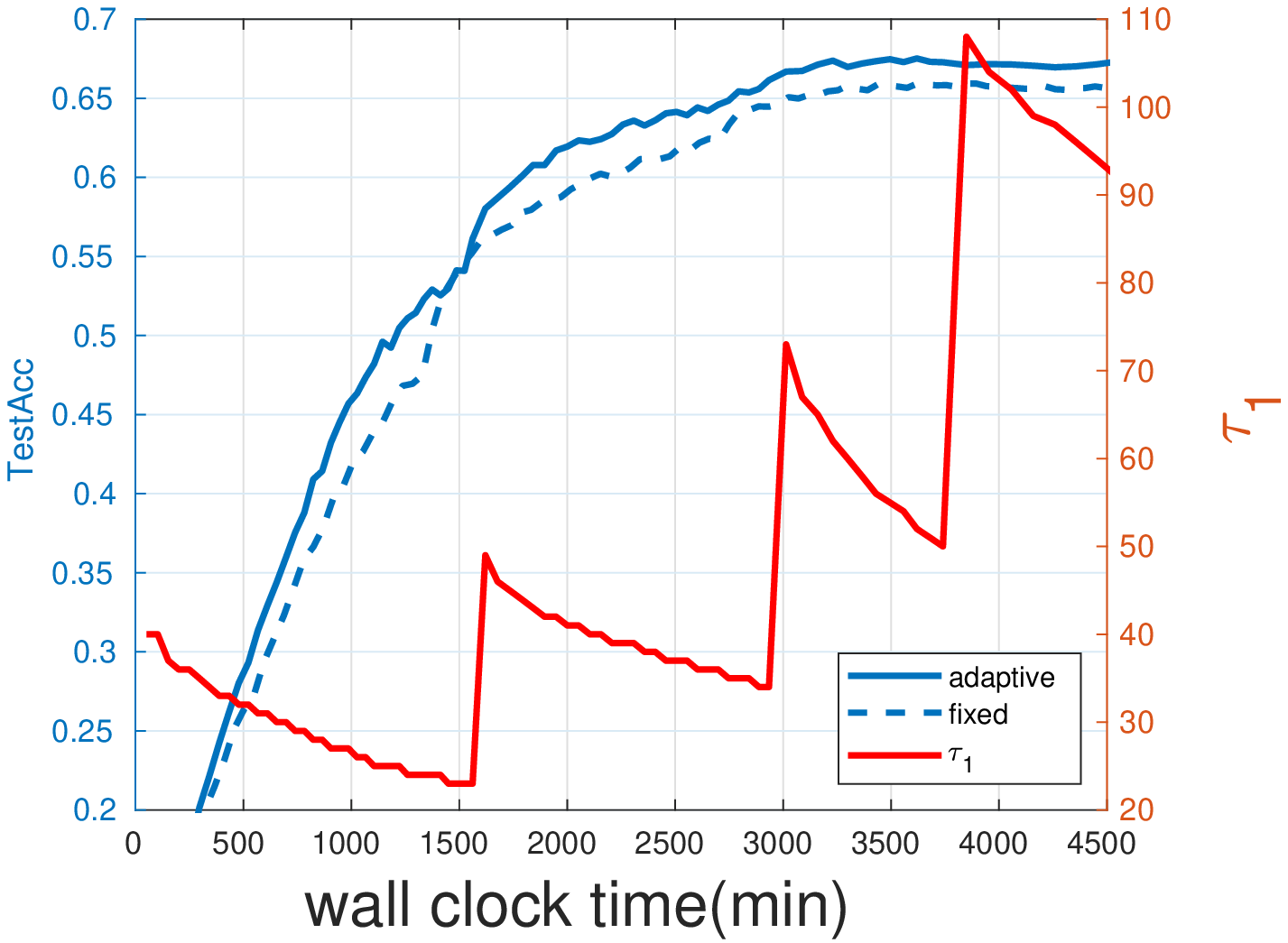}
		\caption{\textit{CIFAR-100}, $\alpha=1$}
		\vspace{20pt}
		\label{figsim4:sub2}
	\end{subfigure}
	\begin{subfigure}{.33\textwidth}
		\captionsetup{width=0.8\textwidth}
		\centering
		\includegraphics[width=.9\linewidth]{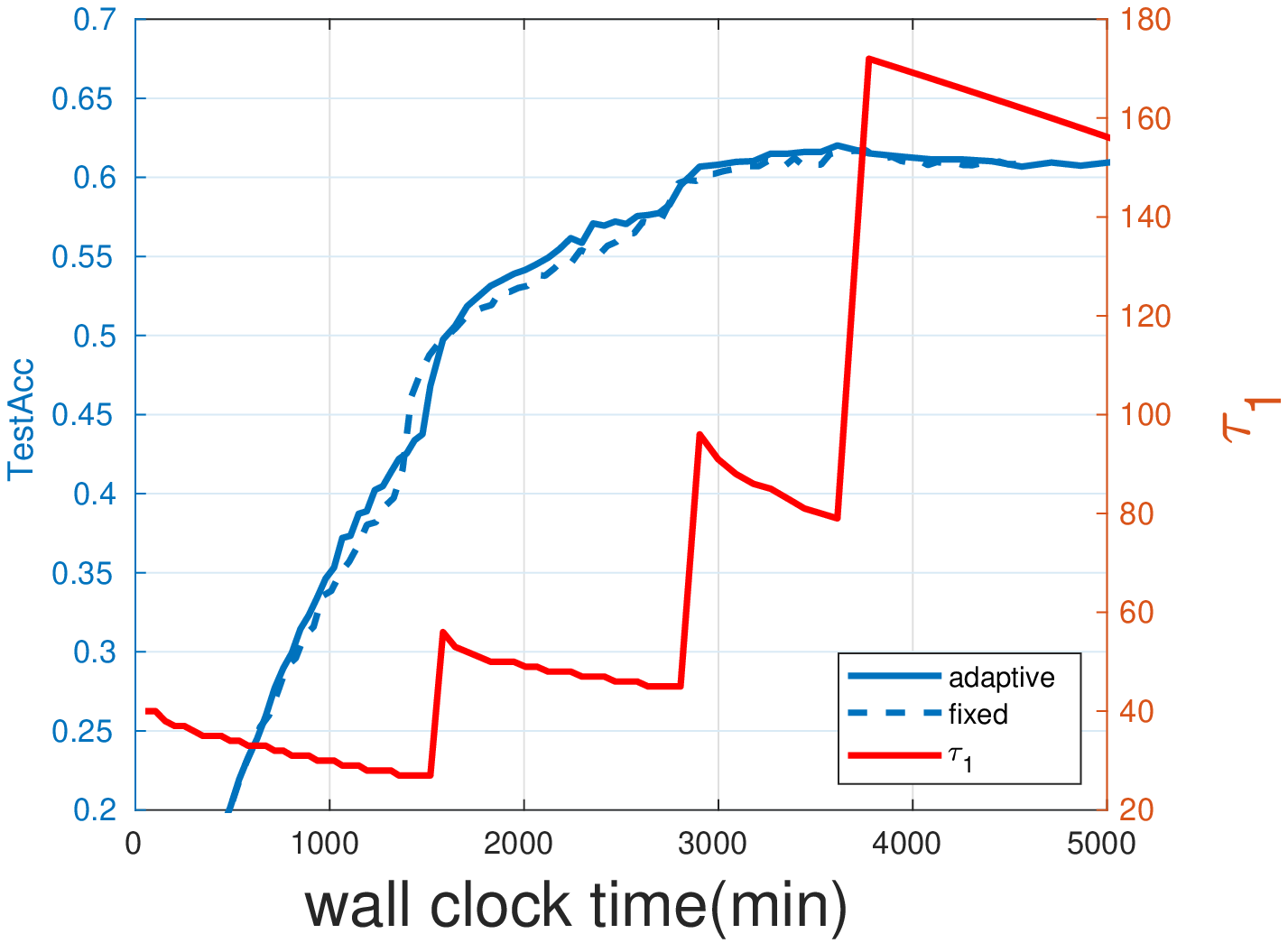}
		\caption{\textit{CIFAR-100}, $ \alpha=0.1$}
		\vspace{20pt}
		\label{figsim4:sub3}
	\end{subfigure}
	\vspace{10pt}
	\caption{{Test accuracy and local aggregation interval $\tau_1$ values against the wall-clock time for adaptive aggregation interval control for different data distributions $\alpha$. $n=20\text{ and } s=4$.}}
	\label{fig:sim3}
\end{figure*}

\section{Numerical Results} \label{Simulations}
In this section, we present sample simulation results for Hier-Local-QSGD to verify the observations from the convergence analysis and illustrate the effectiveness of the proposed adaptive aggregation interval selection algorithm.

\subsection{Settings}
\par
We consider a hierarchical FL system with $n$ clients, $s$ edge servers and a cloud server, assuming that edge server $\ell$ connects with $m^\ell$ clients, each with the same amount of training data. While the design guidelines and the  proposed algorithm is developed for the IID data case, we will also test its effectiveness under the non-IID setting. To simulate different degrees of non-IID data splits, we utilize the Dirichlet distribution $Dir(\alpha)$ as in \cite{NEURIPS2020_18df51b9} with a larger $\alpha$ indicating a more homogeneous data distribution. Particularly in our experiment, $\alpha=100$ represents the IID case, $\alpha=1$ represents the non-IID case, and $\alpha=0.1$ represents the extreme non-IID case. 

For the \textit{CIFAR-10} dataset, we use a CNN with 3 convolutional blocks, which has 5,852,170 parameters and achieves a 90\% testing accuracy in centralized training. For the local computation of the training with \textit{CIFAR-10}, mini-batch SGD is employed with a batch size of 10, an initial learning rate of 0.1, and an exponential learning rate decay of 0.992 for every epoch.  
{For the \text{CIFAR-100}, we use ResNet18, which has 11,220,132 parameters. The local training employs SGD with a batch size of 20. The learning rate is initialized as 0.1 and then set as 0.02, 0.004, and 0.008 at the 60-th, 120-th, and 160-th epoch.}
\par 
For the quantizer, we use the random sparsification operator in Example \ref{example1} and the stochastic rounding in Example \ref{example2}.
The modelling of the computing and communication latency during the training process largely follows \cite{tran2019federated}. We assume homogeneous communication conditions and computing resources for different clients. The clients upload the model through a wireless channel of bandwidth $B$ that is equal to 1Mhz and with a channel gain $h$ that equal to $10^{-8}$. The transmitter power $p$ is fixed at 0.5W, and the noise power $N_0$ is $10^{-10}$ W. For the local computation model, the number of CPU cycles to execute one sample of data is denoted as $c$, which can be measured in advance and is assumed to be 20 cycles/bit in the simulations. The CPU cycle frequency $f$ of the client's device is assumed to be 1 GHz. Thus, the local computation time for one bit of data is $\frac{c}{f}$. We assume the communication latency of the edge-cloud link is 10 times larger than that of the client-edge link. Assume the uploaded model size is $W$ bits, and one local iteration involves $D$ bits of data. In this case, the latency for one full-precision model upload and one local iteration can be calculated with the following equations: $T^{comp} = \frac{cD}{f}, \quad  T^{comm} = \frac{W}{B \log _2(1+\frac{hp}{N_0})}.$

Specifically, the values of the communication and computation latency of the full-precision model mentioned above for \textit{CIFAR-10} with a training batch size of 10 is $T^{comp} = 2s, \text{ and }T^{comm} = 33s.$ The latency for \textit{CIFAR-100} with batch size of 20 is $T^{comp} = 7.2s, \text{ and }T^{comm} = 63.3s.$

\subsection{Verification of the two obtained design guidelines}

\subsubsection{Too much quantization suggests infrequent communication}

In this part, we demonstrate that when the quantization level $q_1$ is above a certain threshold, a larger local update step $\tau_1$ is surprisingly a better choice. 
We simulate a hierarchical system with $n=20$ clients and $s=4$ edge servers. Each edge server serves 5 clients. The aggregation interval product $\tau_1\tau_2$ is fixed to $250$. It can be observed from Eq. \eqref{eq:32} that when the quantization variance parameter $q_1 > n/s - 1 = 4$, then a smaller $\tau_1$ will lead to a faster convergence speed. We adopt the random sparsification operator in Example \ref{example1} as the quantization technique and run experiments with different values of $q_1$. The results are presented in Table. \ref{table2}. 

\par
It can be clearly seen that as the value of $q_1$ increases, the influence of the local update step $\tau_1$ flips. When $q_1=0$, which means no quantization, $\tau_1=5$ leads to the fastest convergence speed. When $q_1=19$, different values of $\tau_1$ have little impact on the convergence.  When $q_1=65.67$, $\tau_1=125$ achieves the best performance, which agrees with the theoretical results. This result is very interesting and promising since infrequent communication and quantization are two important techniques to improve the communication efficiency in FL, but the communication reduction always comes at the price of slower convergence. Our results show both theoretically and experimentally that we can improve the communication efficiency without hurting the convergence speed under certain circumstances. However, it is also noted that the flipping threshold does not match the theoretical result when $q_1=\frac{n}{s}-1=4$. {Simulation results with non-IID data are also presented in Table \ref{table2}. It is observed that the conclusion still holds for the minor non-IID case, i.e., $\alpha=1$. For the extreme non-IID case, when $q_1=65.57$, a small $\tau_1$ may fail the training. The same conclusion is verified when stochastic rounding is adopted as the quantization scheme with IID data and the results are shown in Table \ref{table3}.}
 
\subsubsection{Edge-client Association}


In this part, we show that when using the proposed training algorithm Hier-Local-QSGD, the learning performance is irrelevant to the edge-client association strategy. As revealed in our analysis, this is because the weighted average scheme can balance the variance. In the experiments, we consider a hierarchical system with 20 clients and 2 edge servers. The test accuracy of  three different client-edge association cases, i.e., $(10,10)$,  $(15,5)$ and $(18,2)$, are compared, where $(10,10) $ indicates that 10 clients are associated with the first edge server and the other 10 with the second. It is shown from Fig. \ref{figsim1:sub1} that the test accuracy learning curves of the three association cases coincide with each other using the proposed Hier-Local-QSGD algorithm., which verifies our observation in Section \ref{Convergence Analysis}. Note that the performance curves for the three association strategies diverge when using the uniform average in \cite{wang2020local} at the aggregation step, as shown in Fig. \ref{figsim1:sub2}.  It can be seen that $(10,10)$, in which case the uniform average is equivalent to the weighted average, outperforms the others. {Simulation results with extreme non-IID ($\alpha=0.1$) data split are also presented in Fig. \ref{figsim1:sub3} and Fig. \ref{figsim1:sub4}.} 
Surprisingly, for this particular non-iid case, the edge-client association does not affect the performance. Our analysis cannot provide an explanation for this observation, and we empirically find that this conclusion does not generally hold for non-iid clients. Thus, it is interesting to characterize the conditions under which edge-client association does not have any impact on the training performance.

\subsection{Adaptive Aggregation Interval Control}

We  now evaluate the adaptive aggregation interval control scheme. 
We simulate a hierarchical FL system with $n=20 \text{ and }s=4$. As mentioned in the communication model part, we assume $D_{ec} = 10D_{de}$, and thus we can get the optimal value of $\tau_2 = 7$ using Eq. \eqref{eq:30}. We pick the initial value of $\tau_1$ as 100 for \textit{CIFAR-10} 
{and initial value of $\tau_1$ as 40 for \textit{CIFAR-100}.} The settings for the fixed aggregation interval strategy is $\tau_1=50 \text{ and } \tau_2=5$ for \textit{CIFAR-10} and 
{
$\tau_1=5 \text{ and } \tau_2=50$ for \textit{CIFAR-100}. }
Since we adopt a learning rate decay scheme for the optimization, the update rule for $\tau_1$ in Eq. \eqref{eq:30} will be modified as $\tau_1^j =  \left\lceil \sqrt{\frac{\eta_{t=0}}{\eta_{t=jT_0}}\frac{F(w_{t=jT_0})}{F(w_{t=0})}}\tau_1^0 \right\rceil$.
\par
As shown in Fig. \ref{fig:sim3}, for different data distribution cases, the adaptive aggregation interval selection scheme outperforms the fixed case. The local aggregation interval $\tau_1$ for $\alpha=100$ and $\alpha=1$ decreases during the training process. For $\alpha=0.1$, the training loss will plateau in the latter part of the training but the learning rate is still decaying and the value of $\tau_1$ will increase. 
{For \textit{CIFAR-100}, $\tau_1$ increased when the learnig rate decays at 60-th, 120-th and 160-th epoch.
}

\section{Conclusions}
\par
This paper developed a provably communication-efficient hierarchical FL algorithm and provided a tighter convergence analysis to support system design. The analysis in this paper improves the error term of local update from quadratic to linear, based on which two important design guidelines regarding the local update interval and the edge-client association were provided. Besides, an adaptive algorithm was developed to determine the two values of aggregation intervals. Simulations verified the design guidelines and demonstrated the effectiveness of the adaptive interval selection scheme. One limitation of the current analysis is that it focused on the IID-data, which leads to the conclusion that the edge-client association has no influence on the convergence speed. It would be interesting to extend the analysis of Hier-Local-QSGD to the non-IID case and further investigate the edge-client problem. {In the future work, we also intend to consider the resource allocation problem based on the convergence analysis in realistic wireless FL systems, and considering practical heterogeneous clients.}


%

\appendix
\centerline{\textbf{Proofs of Key Lemmas}}
\section{Proofs of Key Lemmas} \label{Appendix}

\noindent
{\textbf{Lemma \ref{lemma1}}:}
	The proof directly follows from the property of the $L$-smoothness:
	
	\begin{equation}
		f(x) \leq f(y) + \left \langle \nabla f(y), x-y \right \rangle + \frac{L}{2} \left \Vert x-y \right \Vert ^2. \label{eq:13}
	\end{equation}

	For any $L$-smooth function $f$ and variables $x, y$,  then under the $L$-smooth assumption:

	\begin{equation}
		f(x_{k+1}) \leq f(\bar{x}_{k+1}) + \left \langle \nabla f(\bar{x}_{k+1}), x_{k+1}-\bar{x}_{k+1} \right\rangle + \frac{L}{2} \left \Vert x_{k+1}-\bar{x}_{k+1} \right \Vert ^2. \label{eq:14}
	\end{equation}
	\begin{equation}
	f(\bar{x}_{k+1}) \leq f(x_k) + \left \langle \nabla f(x_k, \bar{x}_{k+1}- x_k \right \rangle + \frac{L}{2} \left \Vert \bar{x}_{k+1}- x_k \right \Vert ^2. \label{eq:15}
	\end{equation}

	By taking the expectations of both sides of \eqref{eq:14} and from Eqns. \eqref{eq:10}, \eqref{eq:11} and the unbiased assumption of the random quantizer $Q_2$, we have $\mathbb{E}_{Q_2} [x_{k+1}] = \bar{x}_{k+1}$, so that \eqref{eq:14} becomes: 
	\begin{equation}
		\mathbb{E} f(x_{k+1}) \leq  \mathbb{E} f(\bar{x}_{k+1}) + \frac{L}{2} \mathbb{E} \left \Vert x_{k+1}-\bar{x}_{k+1} \right \Vert ^2  
		\label{eq:16}	
	\end{equation} 

	Similarly, by taking the expectation over \eqref{eq:15} and combining it with \eqref{eq:16}, Lemma \ref{lemma1} is proved.

\vspace{5mm}
\noindent
{\textbf{Lemma \ref{lemma2}}:}
	From Eqn.  \eqref{eq:10}, we have
	\small
	\begin{equation}
		\bar{x}_{k+1} - x_k = - \eta \sum_{\ell \in [s]} \frac{m^\ell}{n} 
		\frac{1}{m^\ell}\sum_{\alpha=0}^{\tau_2-1}\sum_{j\in \mathcal{C}^{\ell}} Q_1^{(\alpha)}\left [ \sum_{\beta=0}^{\tau_1-1} \tilde{\nabla}f_j(x^j_{k,\alpha,\beta}) \right ]  \label{eq:17}
	\end{equation}
	\normalsize
	Then by taking the expectation and changing the subscript from $(\alpha, \beta)$ to $(t_2, t_1)$, we obtain: 
	\small
	\begin{align}
		&\mathbb{E} \langle \nabla f(x_k), \bar{x}_{k+1} - x_k \rangle  \notag \\
		&= - \mathbb{E} \langle \nabla f(x_k), \eta \sum_{\ell \in [s]} \frac{m^\ell}{n} 
		\frac{1}{m^\ell}\sum_{\alpha=0}^{\tau_2-1} \sum_{j\in \mathcal{C}^{\ell}} Q_1^{(\alpha)}\left [ \sum_{\beta=0}^{\tau_1-1} \tilde{\nabla}f_j(x^j_{k,\alpha,\beta}) \right ] \rangle  \\
		& = - \frac{\eta}{n} \sum_{j\in [n]} 
		\sum_{t_2=0}^{\tau_2-1}  \sum_{t_1=0}^{\tau_1-1}  \mathbb{E}_{
			 \left\{ Q_1 ^{(\alpha)},\{ \xi_{k,\alpha,\beta} \}_{\beta = 0}^{\tau_1-1} \right\}_{\alpha = 0}^{t_2-1}, \atop
			 \{ \xi_{k,t_2,\beta} \}_{\beta = 0}^{t_1-1} }
		  {\langle \nabla f(x_k), {\nabla}f(x^j_{k,t_2,t_1}) \rangle} \label{eq:18}
	\end{align}
	\normalsize
	Here for each tuple of $(k,t_2, t_1)$, $\mathbb{E}$ means taking the expectation of the randomness generated from the SGD and the quantization occurred before step $(k,t_2, t_1)$.
	For the simplicity of notation, we omit the subscript of the expectation operation and use $\mathbb{E}$. Using the identity $2\langle \boldsymbol{a}, \boldsymbol{b} \rangle = \Vert \boldsymbol{a} \Vert ^2 + \Vert \boldsymbol{b} \Vert ^2 - \Vert \boldsymbol{a} -\boldsymbol{b} \Vert ^2$, we have
	\small
	\begin{align}
		- \mathbb{E} \langle \nabla f(x_k), {\nabla}f(x^j_{k,t_2,t_1}) \rangle 
		 =- \frac{1}{2} &\mathbb{E} \left \Vert \nabla f(x_k) \right \Vert ^2 -\frac{1}{2} \mathbb{E} \left \Vert {\nabla}f(x^j_{k,t_2,t_1}) \right \Vert ^2  \notag \\
		  & +\frac{1}{2} \mathbb{E} \left \Vert { \nabla f(x_k) - \nabla}f(x^j_{k,t_2,t_1}) \right \Vert ^2 \label{eq:19}
	\end{align}
	\normalsize
	Now, we will bound the third term on the RHS of \eqref{eq:19}
	\small
	
		\begin{align}
			&\mathbb{E} \left \Vert  \nabla f(x_k) - {\nabla}f(x^j_{k,t_2,t_1}) \right \Vert ^2  \notag \\
			=& L^2 \eta ^2 \mathbb{E} \left \Vert  \sum_{\beta=0}^{t_1-1} \tilde{\nabla}f_i(x^i_{k,t_2,\beta}) + \sum_{\alpha=0}^{t_2-1}\sum_{j\in \mathcal{C}^{\ell _{i}}} \frac{1}{m^{\ell_i}} Q_1^{(\alpha)}\left [ \sum_{\beta=0}^{\tau_1-1} \tilde{\nabla}f_j(x^j_{k,\alpha,\beta})\right]  \right \Vert ^2  
			\label{eq:22}
		\end{align}
	
	\normalsize
	Then by using $\mathbb{E} \Vert x \Vert ^2 =  \left \Vert \mathbb{E}x \right\Vert ^2 + Var (x)^2$, we get
	\small
	\begin{align}
	&\mathbb{E} \left \Vert  \sum_{\beta=0}^{t_1-1} \tilde{\nabla}f_i(x^i_{k,t_2,\beta}) + \sum_{\alpha=0}^{t_2-1}\sum_{j\in \mathcal{C}^{\ell _{i}}} \frac{1}{m^{\ell_i}} Q_1^{(\alpha)}\left [ \sum_{\beta=0}^{\tau_1-1} \tilde{\nabla}f_j(x^j_{k,\alpha,\beta})\right]  \right \Vert ^2  \\
	= & \underbrace{\mathbb{E} \left \Vert  \sum_{\beta=0}^{t_1-1} {\nabla}f(x^i_{k,t_2,\beta}) + \sum_{\alpha=0}^{t_2-1}\sum_{j\in \mathcal{C}^{\ell _{i}}} \frac{1}{m^{\ell_i}} \left [ \sum_{\beta=0}^{\tau_1-1}  {\nabla}f(x^j_{k,\alpha,\beta})\right]  \right \Vert ^2}_{A} \notag  \\
	& + \underbrace{\mathbb{E} \left \Vert  \sum_{\beta=0}^{t_1-1} \left[\tilde{\nabla}f_i(x^i_{k,t_2,\beta})- {\nabla}f(x^i_{k,t_2,\beta})\right] \atop + \sum_{\alpha=0}^{t_2-1}\sum_{j\in \mathcal{C}^{\ell _{i}}} \frac{1}{m^{\ell_i}}  \left\{ Q_1^{(\alpha)}\left [ \sum_{\beta=0}^{\tau_1-1} \tilde{\nabla}f_j(x^j_{k,\alpha,\beta})\right] \atop
	- \sum_{\beta=0}^{\tau_1-1}  {\nabla}f(x^j_{k,\alpha,\beta})  \right\}  \right \Vert ^2}_{B}  \\
	= & \underbrace{\mathbb{E} \left \Vert  \sum_{\beta=0}^{t_1-1} {\nabla}f(x^i_{k,t_2,\beta}) + \sum_{\alpha=0}^{t_2-1}\sum_{j\in \mathcal{C}^{\ell _{i}}} \frac{1}{m^{\ell_i}} \left [ \sum_{\beta=0}^{\tau_1-1}  {\nabla}f(x^j_{k,\alpha,\beta})\right]  \right \Vert ^2}_A \notag \\
	 &+ \underbrace{\sum_{\beta=0}^{t_1-1} \mathbb{E} \left \Vert   \left[\tilde{\nabla}f_i(x^i_{k,t_2,\beta})- {\nabla}f(x^i_{k,t_2,\beta})\right] \right \Vert ^2}_{B_1} \notag \\
	 & + \underbrace{\sum_{\alpha=0}^{t_2-1} \mathbb{E} \left \Vert \sum_{j\in \mathcal{C}^{\ell _{i}}} \frac{1}{m^{\ell_i}}  \left\{ Q_1^{(\alpha)}\left [ \sum_{\beta=0}^{\tau_1-1} \tilde{\nabla}f_j(x^j_{k,\alpha,\beta})\right]- \sum_{\beta=0}^{\tau_1-1}  {\nabla}f(x^j_{k,\alpha,\beta})  \right\}  \right \Vert ^2}_{B_2} \label{eq:20}
	\end{align}
	\normalsize
The expectation did not disappear on term A since the sequence $\{x_{k,\alpha, \beta}^j\}$ itself is also random. The inner product vanished when $B$ is expanded, and the conditional independence is used, the detailed proof process are shown on the top of next page.
	\begin{figure*}
	Showing $B=0$:
	\small
	\begin{enumerate}
		\item $\mathbb{E} \left \langle \nabla f(x_{k,t_2,s}^j)-\tilde{\nabla} f_j(x_{k,t_2,s}^j), \nabla f(x_{k,t_2,s}^j)-\tilde{\nabla} f_j(x_{k,t_2,s}^j) \right \rangle$,$ \forall s < t$
			\begin{flalign*}
				& \mathbb{E}_{x_{k,t_2,s}^j,\xi_{k,t_2,s}^j,x_{k,t_2,t}^j,\xi_{k,t_2,t}^j} \left \langle \nabla f(x_{k,t_2,s}^j)-\tilde{\nabla} f_j(x_{k,t_2,s}^j), \nabla f(x_{k,t_2,s}^j)-\tilde{\nabla} f_j(x_{k,t_2,s}^j) \right \rangle &\\
				= & \mathbb{E}_{x_{k,t_2,s}^j,\xi_{k,t_2,s}^j,x_{k,t_2,t}^j} \left \langle \nabla f(x_{k,t_2,s}^j)-\tilde{\nabla} f_j(x_{k,t_2,s}^j), \mathbb{E}_{\xi_{k,t_2,t}^j} \left[ \nabla f(x_{k,t_2,s}^j)-\tilde{\nabla} f_j(x_{k,t_2,s}^j) \right] \right \rangle = 0
			\end{flalign*}
		
		\small
		\item $\mathbb{E} \left \langle \nabla f(x_{k,t_2,s}^j)-\tilde{\nabla} f_j(x_{k,t_2,s}^j), \sum_{j\in \mathcal{C}^{\ell _{i}}} \left\{ Q_1^{(\alpha)}\left [ \sum_{\beta=0}^{\tau_1-1} \tilde{\nabla}f_j(x^j_{k,\alpha,\beta})\right]- \sum_{\beta=0}^{\tau_1-1}  {\nabla}f(x^j_{k,\alpha,\beta})  \right\} \right \rangle$: 
		\small
		\begin{flalign*}
			& \mathbb{E}_{x_{k,t_2,s}^j,\xi_{k,t_2,s}^j,\{x_{k,\alpha,\beta}^j\}_{\beta=0}^{\tau_1},\{\xi_{k,\alpha,\beta}^j\}_{\beta=0}^{\tau_1} } \left \langle \nabla f(x_{k,t_2,s}^j)-\tilde{\nabla} f_j(x_{k,t_2,s}^j), \sum_{j\in \mathcal{C}^{\ell _{i}}} \left\{ Q_1^{(\alpha)}\left [ \sum_{\beta=0}^{\tau_1-1} \tilde{\nabla}f_j(x^j_{k,\alpha,\beta})\right]- \sum_{\beta=0}^{\tau_1-1}  {\nabla}f(x^j_{k,\alpha,\beta})  \right\} \right \rangle &\\
			=&  \mathbb{E}_{x_{k,t_2,s}^j,\{x_{k,\alpha,\beta}^j\}_{\beta=0}^{\tau_1},\{\xi_{k,\alpha,\beta}^j\}_{\beta=0}^{\tau_1}} \left \langle \mathbb{E}_{\xi_{k,t_2,s}^j}  [\nabla f(x_{k,t_2,s}^j)-\tilde{\nabla} f_j(x_{k,t_2,s}^j)], \sum_{j\in \mathcal{C}^{\ell _{i}}} \left\{ Q_1^{(\alpha)}\left [ \sum_{\beta=0}^{\tau_1-1} \tilde{\nabla}f_j(x^j_{k,\alpha,\beta})\right]- \sum_{\beta=0}^{\tau_1-1}  {\nabla}f(x^j_{k,\alpha,\beta})  \right\} \right \rangle &\\
			=& \mathbb{E}_{x_{k,t_2,s}^j,\{x_{k,\alpha,\beta}^j\}_{\beta=0}^{\tau_1},\{\xi_{k,\alpha,\beta}^j\}_{\beta=0}^{\tau_1}} \left \langle 0, \sum_{j\in \mathcal{C}^{\ell _{i}}} \left\{ Q_1^{(\alpha)}\left [ \sum_{\beta=0}^{\tau_1-1} \tilde{\nabla}f_j(x^j_{k,\alpha,\beta})\right]- \sum_{\beta=0}^{\tau_1-1}  {\nabla}f(x^j_{k,\alpha,\beta})  \right\} \right \rangle =0
		\end{flalign*}
		
		\item \small$\forall s< t, \mathbb{E} \left \langle \sum_{j\in \mathcal{C}^{\ell _{i}}} \left\{ Q_1^{(s)}\left [ \sum_{\beta=0}^{\tau_1-1} \tilde{\nabla}f_j(x^j_{k,s,\beta})\right]- \sum_{\beta=0}^{\tau_1-1}  {\nabla}f(x^j_{k,s,\beta})  \right\},
		\sum_{j\in \mathcal{C}^{\ell _{i}}} \left\{ Q_1^{(t)}\left [ \sum_{\beta=0}^{\tau_1-1} \tilde{\nabla}f_j(x^j_{k,t,\beta})\right]- \sum_{\beta=0}^{\tau_1-1}  {\nabla}f(x^j_{k,t,\beta})  \right\} \right \rangle$.
		\normalsize 
		\vspace{2.5mm}
		The expansion of this form of the inner product is very complex, but it can still be proved to be $0$ in the expectation following the same idea of using the conditional independence as in the above mentioned two cases.
	\end{enumerate}
	\vspace{2.5mm}
	\hrule
\end{figure*}

	\normalsize
	Now we continue to bound the three terms $A, B_1,\text{ and } B_2$ in \eqref{eq:20} as
	\small
	\begin{align}
	A & \leq t_1 \sum_{\beta=0}^{t_1 - 1} \mathbb{E} \left\Vert \nabla f(x_{k,t_2, \beta}) \right\Vert^2  \notag\\
	&+ \frac{1}{m^{\ell_{i}}} t_2\tau_1\sum_{i \in \mathcal{C}^\ell_i} \sum_{\alpha=0}^{t_2-1}\sum_{\beta = 0}^{\tau_1-1} \mathbb{E}\left\Vert \nabla f(x_{k,\alpha, \beta}^j) \right\Vert^2
	\end{align}
	\begin{align}
		B_1 = \sum_{\beta=0}^{t_1-1} \mathbb{E} \left \Vert   \left[\tilde{\nabla}f_i(x^i_{k,t_2,\beta})- {\nabla}f(x^i_{k,t_2,\beta})\right] \right \Vert ^2 \leq t_1\sigma^2 \label{eq:21}
	\end{align}
	\normalsize
	Note that \eqref{eq:21} directly follows from Assumption \ref{assumption2}. Next
	\small
	\begin{align}
		&B_2 \notag \\
		= &\sum_{\alpha=0}^{t_2-1} \mathbb{E} \left \Vert \sum_{j\in \mathcal{C}^{\ell _{i}}} \frac{1}{m^{\ell_i}}  \left\{ Q_1^{(\alpha)}\left [ \sum_{\beta=0}^{\tau_1-1} \tilde{\nabla}f_j(x^j_{k,\alpha,\beta})\right]- \sum_{\beta=0}^{\tau_1-1}  {\nabla}f(x^j_{k,\alpha,\beta})  \right\}  \right \Vert ^2 \label{eq:33} \\
		 = &\sum_{\alpha=0}^{t_2-1} \sum_{j\in \mathcal{C}^{\ell _{i}}} \frac{\mathbb{E} \left \Vert \left\{ Q_1^{(\alpha)}\left [ \sum_{\beta=0}^{\tau_1-1} \tilde{\nabla}f_j(x^j_{k,\alpha,\beta})\right]- \sum_{\beta=0}^{\tau_1-1}  {\nabla}f(x^j_{k,\alpha,\beta})  \right\}  \right \Vert ^2}{\left(m^{\ell_i}\right)^2}  \\
		 =& \sum_{\alpha=0}^{t_2-1} \sum_{j\in \mathcal{C}^{\ell _{i}}} \frac{1}{\left(m^{\ell_i}\right)^2} \mathbb{E} \left\{ \left \Vert \left\{ Q_1^{(\alpha)}\left [ \sum_{\beta=0}^{\tau_1-1} \tilde{\nabla}f_j(x^j_{k,\alpha,\beta})\right] \atop - \sum_{\beta=0}^{\tau_1-1}  \tilde{\nabla}f_j(x^j_{k,\alpha,\beta})  \right\}  \right \Vert ^2 \atop +  \left \Vert \left\{ \sum_{\beta=0}^{\tau_1-1} \tilde{\nabla}f_j(x^j_{k,\alpha,\beta}) \atop - \sum_{\beta=0}^{\tau_1-1}  {\nabla}f(x^j_{k,\alpha,\beta})  \right\}  \right \Vert ^2 \right \} \\
		 \leq &\sum_{\alpha=0}^{t_2-1} \sum_{j\in \mathcal{C}^{\ell _{i}}} \frac{1}{\left(m^{\ell_i}\right)^2} 
		\left \{
				q_1 \mathbb{E} \left\Vert \sum_{\beta=0}^{\tau_1-1} \tilde{\nabla}f_j(x^j_{k,\alpha,\beta}) \right\Vert^2 + \tau_1 \sigma^2
		 \right \} \\
		 \leq& \sum_{\alpha=0}^{t_2-1} \sum_{j\in \mathcal{C}^{\ell _{i}}} \frac{1}{\left(m^{\ell_i}\right)^2} 
		\left \{
		q_1 \mathbb{E} \left\Vert \sum_{\beta=0}^{\tau_1-1} {\nabla}f(x^j_{k,\alpha,\beta}) \right\Vert^2 
		+ q_1\tau_1\sigma^2 \right \} \notag \\
		& + \frac{1}{m^{\ell_i}}t_2\tau_1\sigma^2 \\
		 \leq & \sum_{\alpha=0}^{t_2-1} \sum_{j\in \mathcal{C}^{\ell _{i}}} \frac{1}{\left(m^{\ell_i}\right)^2} 
		\left \{
		q_1\tau_1 \sum_{\beta=0}^{\tau_1-1} \mathbb{E} \left\Vert  {\nabla}f(x^j_{k,\alpha,\beta}) \right\Vert^2
		 \right \} + \frac{1+q_1}{m^{\ell_i}}t_2\tau_1\sigma^2
	\end{align}
	
	\normalsize
	The inner-product vanished when expanding $B_2$ using the same conditional independence mentioned above when expanding term $B$.
	\par
	Now we can write the bound on \eqref{eq:22} as
	\small
	\begin{align}
		& \mathbb{E}\left \Vert  \nabla f(x_k) - {\nabla}f(x^j_{k,t_2,t_1}) \right \Vert ^2 
		 \leq L^2\eta^2 \left(A + B_1+ B_2\right)\\
		 \leq& L^2\eta^2 \left\{ \underbrace{t_1 \sum_{\beta=0}^{t_1 - 1} \mathbb{E} \left\Vert \nabla f(x_{k,t_2, \beta}) \right\Vert^2 \atop 
		 + \frac{1}{m^{\ell_{i}}} \left(t_2\tau_1 +\frac{q_1\tau_1}{m^{\ell_{i}}} \right)\sum_{i \in \mathcal{C}^\ell_i} \sum_{\alpha=0}^{t_2-1}\sum_{\beta = 0}^{\tau_1-1} \mathbb{E}\left\Vert \nabla f(x_{k,\alpha, \beta}^j) \right\Vert^2 }_{C_1} \right\} \atop
		  + L^2\eta^2 \left( \underbrace{t_1\sigma^2 + \frac{1+q_1}{m^{\ell_i}}t_2\tau_1\sigma^2 }_{C_2}\right) 
	\end{align}
	
	\normalsize
	By summing $C_1$ over $i, t_2, \text{ and }t_1$, we then get
	\vspace{6mm}
	\small
	\begin{align}
		&\frac{1}{n} \sum_{i=1}^n \sum_{t_2=0}^{\tau_2-1}\sum_{t_1=0}^{\tau_1-1} C_1 \\
		= & \frac{1}{n}\sum_{i=1}^n \sum_{t_2=0}^{\tau_2-1}\sum_{t_1=0}^{\tau_1-1} \left( t_1 \sum_{\beta=0}^{t_1 - 1} \mathbb{E} \left\Vert \nabla f(x_{k,t_2, \beta}) \right\Vert^2 + \atop
		 \frac{t_2\tau_1 +\frac{q_1\tau_1}{m^{\ell_{i}}}}{m^{\ell_{i}}} \sum_{j \in \mathcal{C}^\ell_i} \sum_{\alpha=0}^{t_2-1}\sum_{\beta = 0}^{\tau_1-1} \mathbb{E}\left\Vert \nabla f(x_{k,\alpha, \beta}^j) \right\Vert^2 \right) \\
		\leq & \frac{(\tau_1-1)\tau_1}{2} \frac{1}{n}\sum_{i=1}^n \sum_{\alpha=0}^{\tau_2-1}\sum_{\beta=0}^{\tau_1-1} \mathbb{E} \left\Vert \nabla f(x_{k,\alpha, \beta}) \right\Vert^2  \notag \\
		&+ \frac{1}{n}\sum_{i=1}^n \sum_{t_2=0}^{\tau_2-1}\sum_{t_1=0}^{\tau_1-1} \left(t_2\tau_1 +q_1\tau_1 \right) \frac{1}{m^{\ell_{i}}} \sum_{j \in \mathcal{C}^\ell_i} \sum_{\alpha=0}^{\tau_2-1}\sum_{\beta = 0}^{\tau_1-1} \mathbb{E}\left\Vert \nabla f(x_{k,\alpha, \beta}^j) \right\Vert^2 \\
		= & \frac{(\tau_1-1)\tau_1}{2} \frac{1}{n}\sum_{i=1}^n \sum_{\alpha=0}^{\tau_2-1}\sum_{\beta=0}^{\tau_1-1} \mathbb{E} \left\Vert \nabla f(x_{k,\alpha, \beta}) \right\Vert^2 \notag \\
		+& \frac{1}{n}\sum_{i=1}^n \frac{1}{m^{\ell_{i}}} \sum_{j \in \mathcal{C}^\ell_i} \sum_{\alpha=0}^{\tau_2-1}\sum_{\beta = 0}^{\tau_1-1} \mathbb{E}\left\Vert \nabla f(x_{k,\alpha, \beta}^j) \right\Vert^2  \sum_{t_2=0}^{\tau_2-1}\sum_{t_1=0}^{\tau_1-1} \left(t_2\tau_1 +q_1\tau_1 \right) \\
		= & \frac{(\tau_1-1)\tau_1}{2} \frac{1}{n}\sum_{i=1}^n \sum_{\alpha=0}^{\tau_2-1}\sum_{\beta=0}^{\tau_1-1} \mathbb{E} \left\Vert \nabla f(x_{k,\alpha, \beta}) \right\Vert^2 \notag \\
		+& \frac{1}{n}\sum_{i=1}^n \sum_{\alpha=0}^{\tau_2-1}\sum_{\beta = 0}^{\tau_1-1} \mathbb{E}\left\Vert \nabla f(x_{k,\alpha, \beta}^i) \right\Vert^2  \tau_1\tau_2\left(\frac{\tau_1(\tau_2-1)}{2}+q_1\tau_2 \right) \\
		= & \left[ \frac{(\tau_1-1)\tau_1}{2} + \atop \tau_1\tau_2\left(\frac{\tau_1(\tau_2-1)}{2}+q_1\tau_2 \right) \right]  \frac{1}{n}\sum_{i=1}^n \sum_{\alpha=0}^{\tau_2-1}\sum_{\beta=0}^{\tau_1-1} \mathbb{E} \left\Vert \nabla f(x_{k,\alpha, \beta}) \right\Vert^2
	\end{align}
	\normalsize
	By summing $C_2$ over $i, t_2,\text{ and } t_1$, we get
	\small
	\begin{align}
		&\frac{1}{n} \sum_{i=1}^n \sum_{t_2=0}^{\tau_2-1}\sum_{t_1=0}^{\tau_1-1} C_2 \\
		= & \frac{1}{n}\sum_{i=1}^n \sum_{t_2=0}^{\tau_2-1}\sum_{t_1=0}^{\tau_1-1} \left( t_1\sigma^2 + \frac{1+q_1}{m^{\ell_i}}t_2\tau_1\sigma^2 \right) \\
		= & \frac{\tau_1 \tau_2}{2} \left( (\tau_1-1) + \frac{s}{n}(1+q_1) \tau_1 (\tau_2-1) \right)\sigma^2
	\end{align} 
	\normalsize
	Finally, we derive the upper bound in the Lemma \ref{lemma2}:
	\small
	\begin{align}
	&\mathbb{E} \langle \nabla f(x_k), \bar{x}_{k+1} - x_k \rangle\\
	\leq& - \frac{\eta}{2} \sum_{j\in [n]} \frac{1}{n} 
	\sum_{t_2=0}^{\tau_2-1}  \sum_{t_1=0}^{\tau_1-1} \mathbb{E} \left \Vert \nabla f(x_k) \right \Vert^2 \notag \\
	&- \frac{\eta}{2} \sum_{j\in [n]} \frac{1}{n} 
	\sum_{t_2=0}^{\tau_2-1}  \sum_{t_1=0}^{\tau_1-1} \mathbb{E} \left \Vert \nabla f(x_{k,t_2,t_1}^i) \right \Vert^2 \notag  \notag \\
	&+ \frac{L^2\eta^3}{2} \sum_{j\in [n]} \frac{1}{n} 
	\sum_{t_2=0}^{\tau_2-1}  \sum_{t_1=0}^{\tau_1-1} \left( C_1 + C_2 \right) \\
	\leq & - \frac{\eta}{2} \sum_{j\in [n]} \frac{1}{n} 
	\sum_{t_2=0}^{\tau_2-1}  \sum_{t_1=0}^{\tau_1-1} \mathbb{E} \left \Vert \nabla f(x_k) \right \Vert^2 \notag \\
	 &- \frac{\eta}{2} \sum_{j\in [n]} \frac{1}{n} 
	\sum_{t_2=0}^{\tau_2-1}  \sum_{t_1=0}^{\tau_1-1} \mathbb{E} \left \Vert \nabla f(x_{k,t_2,t_1}^i) \right \Vert^2 \notag \\
	&+ \frac{L^2\eta^3}{2n} \left\{ \left[ \frac{(\tau_1-1)\tau_1}{2} +  \atop \tau_1\tau_2\left(\frac{\tau_1(\tau_2-1)}{2}+q_1\tau_2 \right) \right]  \sum_{i=1}^n \sum_{\alpha=0}^{\tau_2-1}\sum_{\beta=0}^{\tau_1-1} \mathbb{E} \left\Vert \nabla f(x_{k,\alpha, \beta}^i) \right\Vert^2 \right \} \notag \\
	&+  \frac{L^2\eta^3}{2} \frac{\tau_1 \tau_2}{2} \left[ (\tau_1-1) + \frac{s}{n}(1+q_1) \tau_1 (\tau_2-1) \right] \sigma^2\\
	\leq & - \frac{\eta}{2} \tau_1\tau_2 \mathbb{E} \left \Vert \nabla f(x_k) \right \Vert^2 \notag \\
	&-\frac{\eta}{2n} \left\{ 1-L^2\eta^2\left[\frac{\tau_1(\tau_1-1)}{2} + \atop
	\tau_1\tau_2\left(\frac{\tau_2(\tau_2-1)}{2}+ \atop q_1\tau_2 \right) \right] \right\} \sum_{i=1}^n \sum_{\alpha=0}^{\tau_2-1}\sum_{\beta=0}^{\tau_1-1} \mathbb{E} \left\Vert \nabla f(x_{k,\alpha, \beta}^i) \right\Vert^2 \notag \\
	&+ \frac{L^2\eta^3}{4}\tau_1 \tau_2 \left[ (\tau_1-1) + \frac{s}{n}(1+q_1) \tau_1 (\tau_2-1) \right] \sigma^2.
	\end{align}
	\normalsize

\noindent
 {\textbf{Lemma \ref{lemma3}}:}
	From Eqn. \eqref{eq:10}, using the property: $\mathbb{E} \Vert x \Vert ^2 =  \left \Vert \mathbb{E}x \right\Vert ^2 + Var (x)^2$, we obtain
	\vspace{6mm}
	\small
	\begin{align}
		&\mathbb{E} \left \Vert \bar{x}_{k+1} - x_k \right \Vert ^2 \\
		=& \eta^2 \mathbb{E}\left\Vert \frac{1}{n}\sum_{i\in [n]}\sum_{t_2=0}^{\tau_2-1} Q_1^{(t_2)} \left[ \sum_{t_1=0}^{\tau_1-1} \tilde{ \nabla} f_i(x_{k,t_2,t_1}^i) \right] \right\Vert ^2 \\
		= & \eta^2 \mathbb{E}\left\Vert \frac{1}{n}\sum_{i\in [n]}\sum_{t_2=0}^{\tau_2-1}\sum_{t_1=0}^{\tau_1-1} {\nabla} f(x_{k,t_2,t_1}^i) \right\Vert ^2 \notag \\
		&+  \frac{\eta^2}{n^2}\sum_{i\in [n]}\sum_{t_2=0}^{\tau_2-1}\mathbb{E}\left\Vert Q_1^{(t_2)} \left[ \sum_{t_1=0}^{\tau_1-1} \tilde{ \nabla} f_i(x_{k,t_2,t_1}^i) \right]  - \sum_{t_1=0}^{\tau_1-1} {\nabla} f(x_{k,t_2,t_1}^i) \right\Vert ^2 \\
		\leq & \eta^2 \tau_1\tau_2 \frac{1}{n}\sum_{i\in [n]}\sum_{t_2=0}^{\tau_2-1}\sum_{t_1=0}^{\tau_1-1} \mathbb{E}\left\Vert {\nabla} f(x_{k,t_2,t_1}^i) \right\Vert ^2 \notag \\
		&+ \eta^2 \frac{1}{n^2}\sum_{i\in [n]}\sum_{t_2=0}^{\tau_2-1} \left\{ q_1\mathbb{E} \left\Vert \sum_{t_1=0}^{\tau_1}\nabla f(x_{k,t_2,t_1}^i) \right\Vert ^2 + (1+q_1)\tau_1\sigma^2\right\} \\
		= & \eta^2 \left( \tau_1\tau_2 +\frac{q_1\tau_1}{n} \right)\frac{1}{n}\sum_{i=1}^n \sum_{\alpha=0}^{\tau_2-1}\sum_{\beta=0}^{\tau_1-1} \mathbb{E} \left\Vert \nabla f(x_{k,\alpha, \beta}^i) \right\Vert^2 \notag \\
		&+ \eta^2\frac{1}{n}(1+q_1)\tau_1\tau_2\sigma^2.
	\end{align}
	\normalsize

\noindent
{\textbf{Lemma \ref{lemma4}}:}
	From Eqn. \eqref{eq:10}, \eqref{eq:11}, we know that 
	\small
	\begin{align}
		&\frac{1}{\eta^2} E \left \Vert x_{k+1} -\bar{x}_{k+1} \right \Vert^2 = \notag \\
		\mathbb{E}& \left\Vert  \sum_{\ell \in [s]} \frac{m^\ell}{n} \left\{ Q_2\left[
		\frac{1}{m^\ell}\sum_{\alpha=0}^{\tau_2-1}\sum_{j\in \mathcal{C}^{\ell}} Q_1^{(\alpha)}\left( \sum_{\beta=0}^{\tau_1-1} \tilde{\nabla}f_j(x^j_{k,\alpha,\beta}) \right) \right] \atop -  \left[ \frac{1}{m^\ell}\sum_{\alpha=0}^{\tau_2-1}\sum_{j\in \mathcal{C}^{\ell}} Q_1^{(\alpha)}\left( \sum_{\beta=0}^{\tau_1-1} \tilde{\nabla}f_j(x^j_{k,\alpha,\beta}) \right) \right] \right\}  \right\Vert ^2 \\
	= &  \sum_{\ell \in [s]}\left(\frac{m^\ell}{n}\right)^2 q_2 \mathbb{E} \left\Vert \left[ \frac{1}{m^\ell}\sum_{\alpha=0}^{\tau_2-1}\sum_{j\in \mathcal{C}^{\ell}} Q_1^{(\alpha)}\left(   \sum_{\beta=0}^{\tau_1-1} \tilde{\nabla}f_j(x^j_{k,\alpha,\beta}) \right) \right] \right\Vert ^2 \\
	= & \sum_{\ell \in [s]}\left(\frac{1}{n}\right)^2 q_2 \mathbb{E} \left\Vert \left[ \sum_{\alpha=0}^{\tau_2-1}\sum_{j\in \mathcal{C}^{\ell}} Q_1^{(\alpha)}\left(\sum_{\beta=0}^{\tau_1-1} \tilde{\nabla}f_j(x^j_{k,\alpha,\beta}) \right) \right] \right\Vert ^2 \label{eq:23}
	\end{align}
	\normalsize
	By using $\mathbb{E} \Vert x \Vert ^2 =  \left \Vert \mathbb{E}x \right\Vert ^2 + Var (x)^2$, and following a similar approach as the one in the proof of Lemma \ref{lemma3}, we conclude that
	\small
	\begin{align}
		&\frac{1}{\eta^2}E \left \Vert x_{k+1} -\bar{x}_{k+1} \right \Vert^2 \notag \\
		\leq &  \sum_{\ell \in [s]}\left(\frac{1}{n}\right)^2 q_2  \left[ \left( \tau_1\tau_2 m^\ell+ \atop\tau_1 q_1 \right)\sum_{j \in \mathcal{C}^\ell}\sum_{t_2=0}^{\tau_2-1}\sum_{t_1=0}^{\tau_1-1}\mathbb{E} \left\Vert \nabla f(x_{k,\alpha, \beta}^i) \right\Vert^2  \atop + m^\ell \tau_1\tau_2(1+q_1)\sigma^2 \right]   \\
		\leq & \sum_{\ell \in [s]}\left(\frac{1}{n}\right)^2 q_2  \left[\left( \tau_1\tau_2 n +\atop \tau_1 q \right)\sum_{j \in \mathcal{C}^\ell}\sum_{t_2=0}^{\tau_2-1}\sum_{t_1=0}^{\tau_1-1}\mathbb{E} \left\Vert \nabla f(x_{k,\alpha, \beta}^i) \right\Vert^2  \atop+ m^\ell \tau_1\tau_2(1+q_1)\sigma^2 \right] \\
		\leq & 2 q_2 \left( \tau_1\tau_2 +\frac{q_1\tau_1}{n} \right)\frac{1}{n}\sum_{i=1}^n \sum_{\alpha=0}^{\tau_2-1}\sum_{\beta=0}^{\tau_1-1} \mathbb{E} \left\Vert \nabla f(x_{k,\alpha, \beta}^i) \right\Vert^2 \notag \\
		& + \frac{1}{n}(1+q_1)q_2\tau_1\tau_2\sigma^2
	\end{align}
	\normalsize

\ifCLASSOPTIONcaptionsoff
  \newpage
\fi



%
%
%
\section*{Acknowledgments}
The authors would like to thank anonymous reviewers and the associate editor for their constructive comments.

\bibliographystyle{ieeetran}
\bibliography{ref.bib}

\begin{IEEEbiography}[{\includegraphics[width=1in,height=1.25in,clip,keepaspectratio]{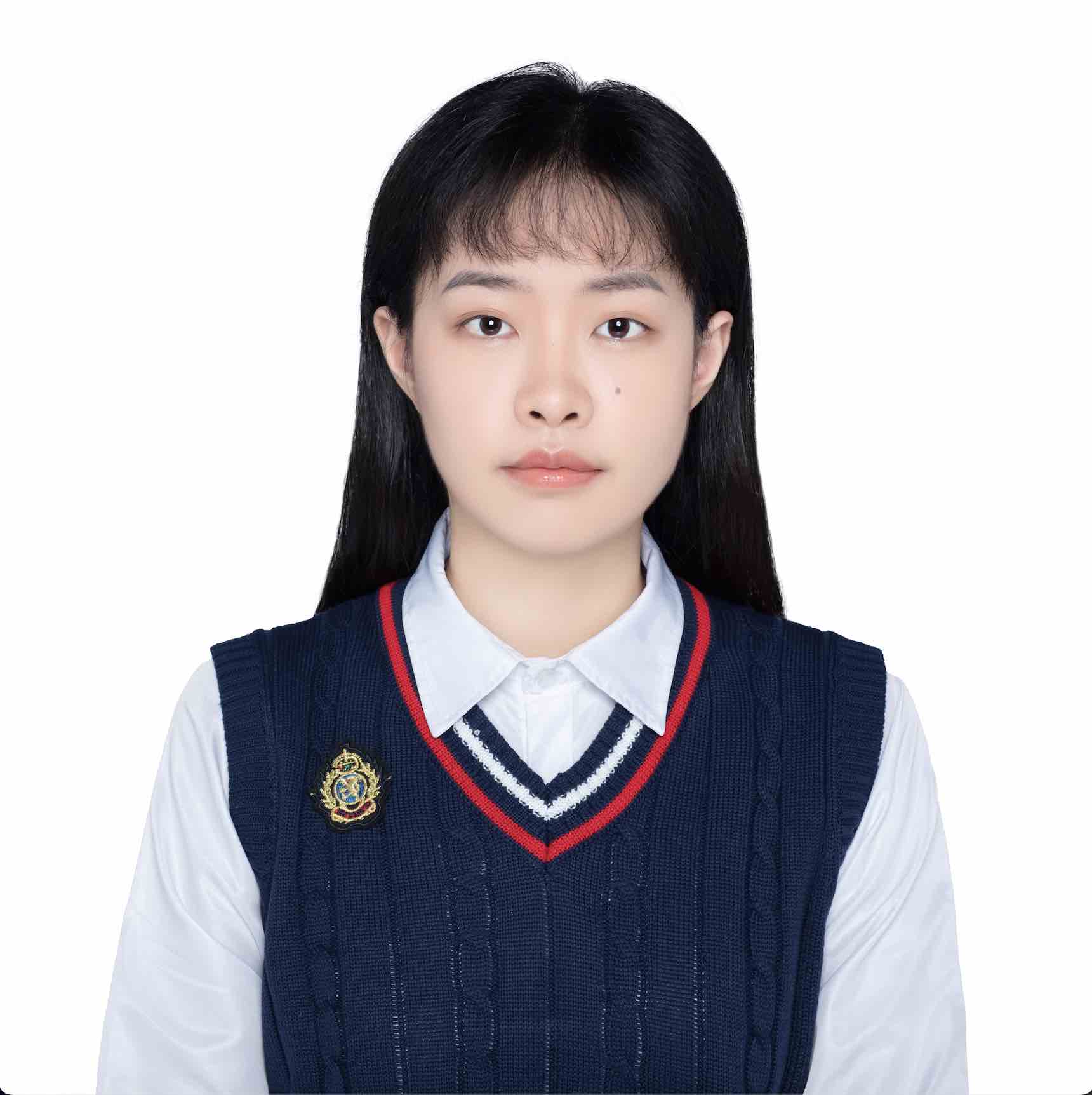}}]{Lumin Liu} (S'18) is currently working towards the Ph.D. degree in the Department of Electronic and Computer Engineering at the Hong Kong University of Science and Technology (HKUST).
\end{IEEEbiography}

\begin{IEEEbiography}[{\includegraphics[width=1in,height=1.25in,clip,keepaspectratio]{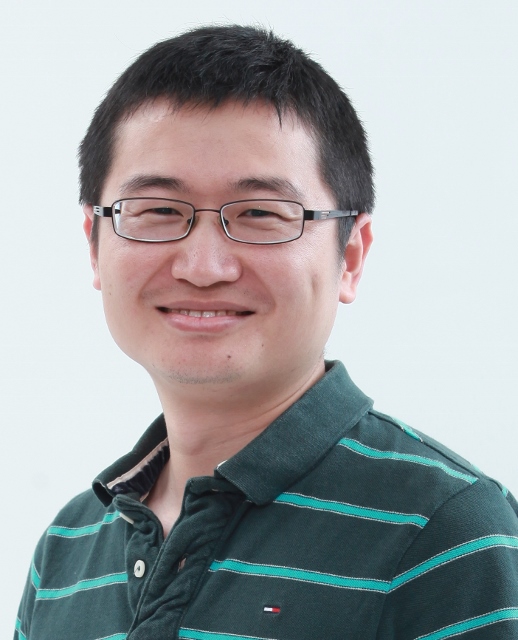}}]{Jun Zhang} (S’06-M’10-SM’15-F'21) 
received the B.Eng. degree in Electronic Engineering from the University of Science and Technology of China in 2004, the M.Phil. degree in Information Engineering from the Chinese University of Hong Kong in 2006, and the Ph.D. degree in Electrical and Computer Engineering from the University of Texas at Austin in 2009. He is an Associate Professor in the Department of Electronic and Computer Engineering at the Hong Kong University of Science and Technology. His research interests include wireless communications and networking, mobile edge computing and edge AI, and cooperative AI.

Dr. Zhang co-authored the book Fundamentals of LTE (Prentice-Hall, 2010). He is a co-recipient of several best paper awards, including the 2021 Best Survey Paper Award of the IEEE Communications Society, the 2019 IEEE Communications Society \& Information Theory Society Joint Paper Award, and the 2016 Marconi Prize Paper Award in Wireless Communications. Two papers he co-authored received the Young Author Best Paper Award of the IEEE Signal Processing Society in 2016 and 2018, respectively. He also received the 2016 IEEE ComSoc Asia-Pacific Best Young Researcher Award. He is an Editor of IEEE Transactions on Communications, and was an editor of IEEE Transactions on Wireless Communications (2015-2020). He served as a MAC track co-chair for IEEE Wireless Communications and Networking Conference (WCNC) 2011 and a co-chair for the Wireless Communications Symposium of IEEE International Conference on Communications (ICC) 2021.

\end{IEEEbiography}

\begin{IEEEbiography}[{\includegraphics[width=1in,height=1.25in,clip,keepaspectratio]{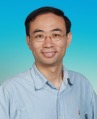}}]{S.H. Song} (S’02–M’06) is now an Assistant Professor jointly appointed by the Division of Integrative Systems and Design (ISD) and the Department of Electronic and Computer Engineering (ECE) at the Hong Kong University of Science and Technology (HKUST). His research is primarily in the areas of Wireless Communications and Machine Learning with current focus on Distributed Intelligence (Federated Learning), Machine Learning for Communications (Model and Data-driven Approaches), and Integrated Sensing and Communication. He was named the Exemplary Reviewer for IEEE Communications Letter and served as the Tutorial Program Co-Chairs of the 2022 IEEE International Mediterranean Conference on Communications and Networking.

Dr. Song is also interested in the research on Engineering Education and is now serving as an Associate Editor for the IEEE Transactions on Education. He has won several teaching awards at HKUST, including the Michael G. Gale Medal for Distinguished Teaching in 2018, the Best Ten Lecturers in 2013, 2015, and 2017, the School of Engineering Distinguished Teaching Award in 2012, the Teachers I Like Award in 2013, 2015, 2016, and 2017, and the MSc(Telecom) Teaching Excellent Appreciation Award for 2020-21. Dr. Song was one of the honorees of the Third Faculty Recognition at HKUST in 2021.

\end{IEEEbiography}

\begin{IEEEbiography}[{\includegraphics[width=1in,height=1.25in,clip,keepaspectratio]{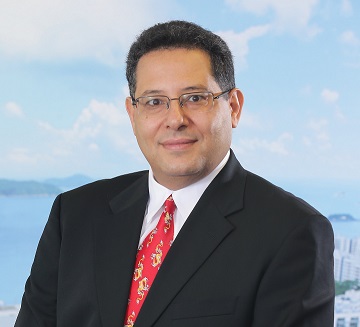}}]{Khaled B. Letaief} (S’85–M’86–SM’97–F’03) received the BS degree with distinction in Electrical Engineering from Purdue University at West Lafayette, Indiana, USA, in December 1984. He received the MS and Ph.D. Degrees in Electrical Engineering from Purdue University, in Aug. 1986, and May 1990, respectively.
	
He is an internationally recognized leader in wireless communications and networks with research interest in wireless communications, artificial intelligence, big data analytics systems, mobile edge computing, 5G systems and beyond.  In these areas, he has over 620 journal and conference papers and given keynote talks as well as courses all over the world.  He also has 15 patents, including 11 US patents.
	
He is well recognized for his dedicated service to professional societies and in particular IEEE where he has served in many leadership positions, including President of IEEE Communications Society, the world’s leading organization for communications professionals with headquarters in New York and members in 162 countries.  He is also the founding Editor-in-Chief of the prestigious IEEE Transactions on Wireless Communications and served on the editorial board of other premier journals including the IEEE Journal on Selected Areas in Communications – Wireless Series (as Editor-in-Chief).  
	
He has been a dedicated teacher committed to excellence in teaching and scholarship with many teaching awards, including the Michael G. Gale Medal for Distinguished Teaching.

He is the recipient of many other distinguished awards including the 2019 Distinguished Research Excellence Award by HKUST School of Engineering (Highest research award and only one recipient/3 years is honored for his/her contributions); 2019 IEEE Communications Society and Information Theory Society Joint Paper Award; 2018 IEEE Signal Processing Society Young Author Best Paper Award; 2017 IEEE Cognitive Networks Technical Committee Publication Award; 2016 IEEE Signal Processing Society Young Author Best Paper Award; 2016 IEEE Marconi Prize Paper Award in Wireless Communications; 2011 IEEE Wireless Communications Technical Committee Recognition Award; 2011 IEEE Communications Society Harold Sobol Award; 2010 Purdue University Outstanding Electrical and Computer Engineer Award; 2009 IEEE Marconi Prize Award in Wireless Communications; 2007 IEEE Communications Society Joseph LoCicero Publications Exemplary Award; and over 15 IEEE Best Paper Awards.
Since 1993, he has been with HKUST where he has held many administrative positions, including the Head of the Electronic and Computer Engineering department.  He also served as Chair Professor and HKUST Dean of Engineering.  Under his leadership, the School of Engineering dazzled in international rankings (rising from $\#$ 26 in 2009 to $\#$ 14 in the world in 2015 according to QS World University Rankings.) 
	
He is a Fellow of IEEE and a Fellow of HKIE. He is also recognized by Thomson Reuters as an ISI Highly Cited Researcher.

\end{IEEEbiography}

\end{document}